\documentclass[research]{fcs}
\usepackage{layout}
\usepackage{subfig}
\usepackage{hyperref}
\usepackage{multirow}
\usepackage{algorithm}
\usepackage{algpseudocode}

\def\c{{\boldsymbol c}}

\def\d{{\boldsymbol d}}

\def\g{{\boldsymbol g}}

\def\l{{\boldsymbol l}}

\def\x{{\boldsymbol x}}

\def\QM{{\mathcal Q}}

\def\TM{{\mathcal T}}

\usepackage{array}
\newcommand{\revised}[1]{\textcolor[rgb]{0,0,0}{#1}}

\title{AOE: Exhaustive Out-of-Distribution Detection via Recalibrating Outlier Labels}

\author[1]{Fengqiang Wan}
\author[1]{Qing-Yuan Jiang}
\author[2]{Fu Shen}
\author[1,+]{Yang Yang}
\address[1]{School of Computer Science and Engineering, Nanjing University of Science and Technology, Nanjing 210094, China.}
\address[2]{Key Laboratory of Modern Agricultural Equipment, Ministry of Agriculture and Rural Affairs, P.R.China.}

\corremail{yyang@njust.edu.cn}
\fcssetup{
  doi = {10.1007/s11704-026-61080-0},
  received = {May 27, 2026},
  accepted = {July 3, 2026},
  paper-id = {261080},
}

\begin{abstract}
Out-of-distribution (OOD) detection is essential for deploying machine learning models in open-world and safety-critical scenarios, where test inputs may deviate from the training distribution and overconfident predictions on unknown samples can lead to unreliable decisions. Outlier Exposure (OE) has emerged as a promising OOD detection paradigm by introducing auxiliary outliers during training to enlarge the margin between in-distribution (ID) and OOD samples. Existing OE-based methods typically enlarge this margin by employing uniform labels to maximize the entropy of OOD samples over ID categories. However, we theoretically show that uniform labels inevitably disregard the relations between OOD samples and ID categories, termed the over-softening effect, leading to a suboptimal margin bound. Our theoretical analysis further reveals that explicitly exploiting such relations can instead yield improved OOD detection performance. Motivated by this insight, we propose \underline{A}daptive Confidence \underline{OE} (AOE), a simple yet effective method that leverages temperature scaling to recalibrate outlier labels. Specifically, AOE generates adaptive soft targets from temperature-scaled model predictions for OOD samples, where the learnable temperature smooths the prediction distribution without fully erasing class-wise relational information. By supervising OOD samples with these adaptive soft targets, AOE preserves the semantic proximity between OOD samples and ID categories while encouraging the softened targets to approach a high-entropy distribution, thereby suppressing overconfident OOD predictions and enlarging the separation margin. Extensive experiments across diverse benchmarks demonstrate the effectiveness of AOE. Compared with OE, AOE consistently lowers FPR95 by 2.40\% and 2.51\% on CIFAR-10, 0.64\% and 9.11\% on CIFAR-100, and 2.02\% and 2.10\% on ImageNet-200 under near-OOD and far-OOD settings, respectively.
\end{abstract}
\keywords{Out-of-distribution, outlier exposure, temperature scaling}

\begin{document}

\section{Introduction}

Machine learning models are typically predicated on the assumption that the test distribution coincides with that of the training distribution. However, this assumption often fails in real-world settings, where models are exposed to out-of-distribution (OOD) samples and suffer significant performance degradation, particularly in safety-critical applications~\cite{MSP:conf/iclr/HendrycksG17,Survey:journals/ijcv/YangZLL24, transferable_fcs_yang, ossl_fcs}. To address this challenge, OOD detection~\cite{DOE:conf/iclr/WangY0DKLH023,MSP:conf/iclr/HendrycksG17,C3OOD:conf/iclr/LeeLLS18,Energy:conf/nips/LiuWOL20,DOE:conf/iclr/WangY0DKLH023} has emerged as an effective solution. It aims to accurately recognize in-distribution (ID) samples while identifying OOD samples \cite{ASH:conf/iclr/DjurisicBAL23, dice:conf/eccv/SunL22a,DBLP:journals/pami/YangJXZ24}.

\begin{figure}[t] 
\begin{minipage}[b]{\linewidth}\centering
\includegraphics[width=.49\linewidth]{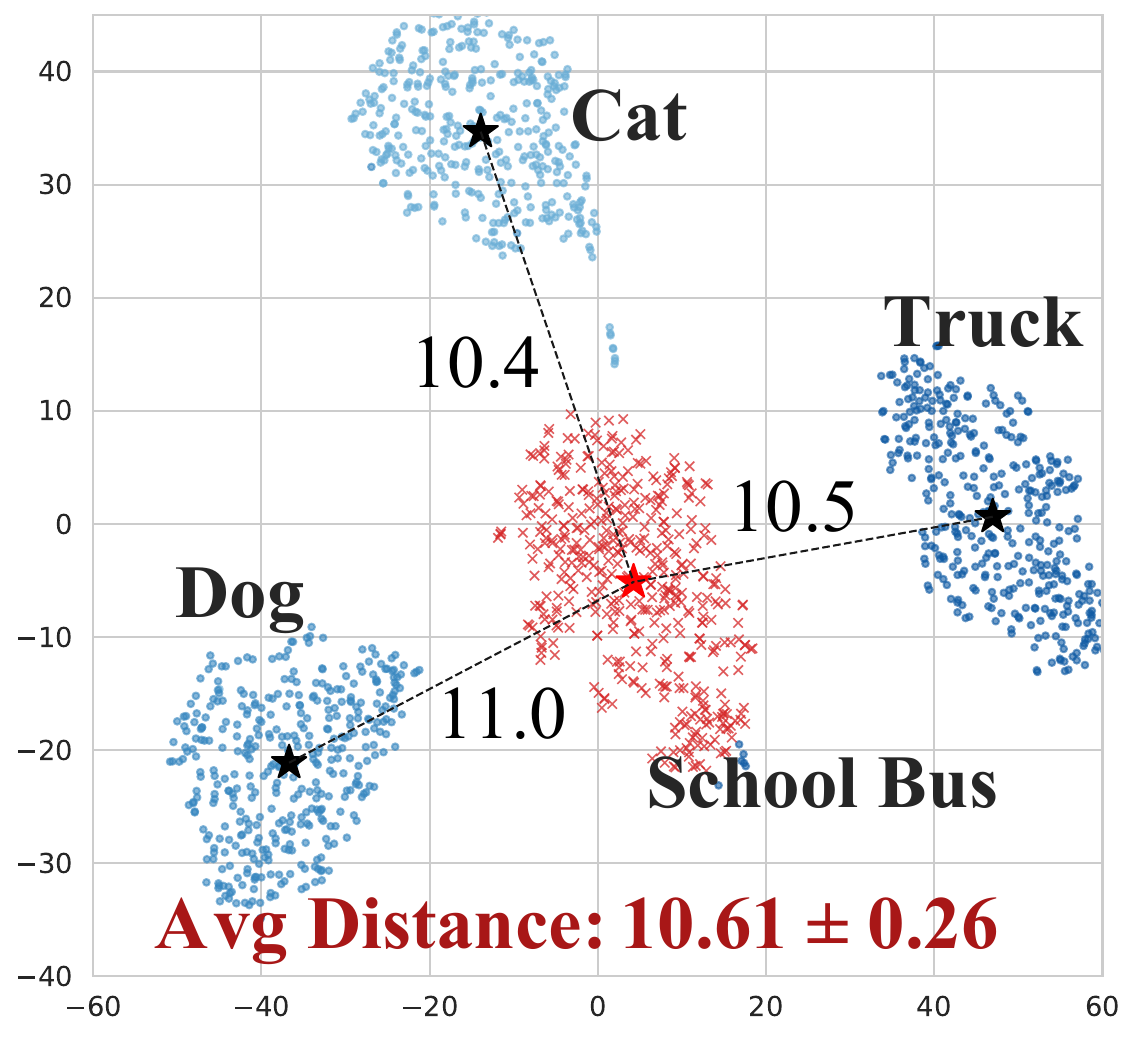}
\includegraphics[width=.49\linewidth]{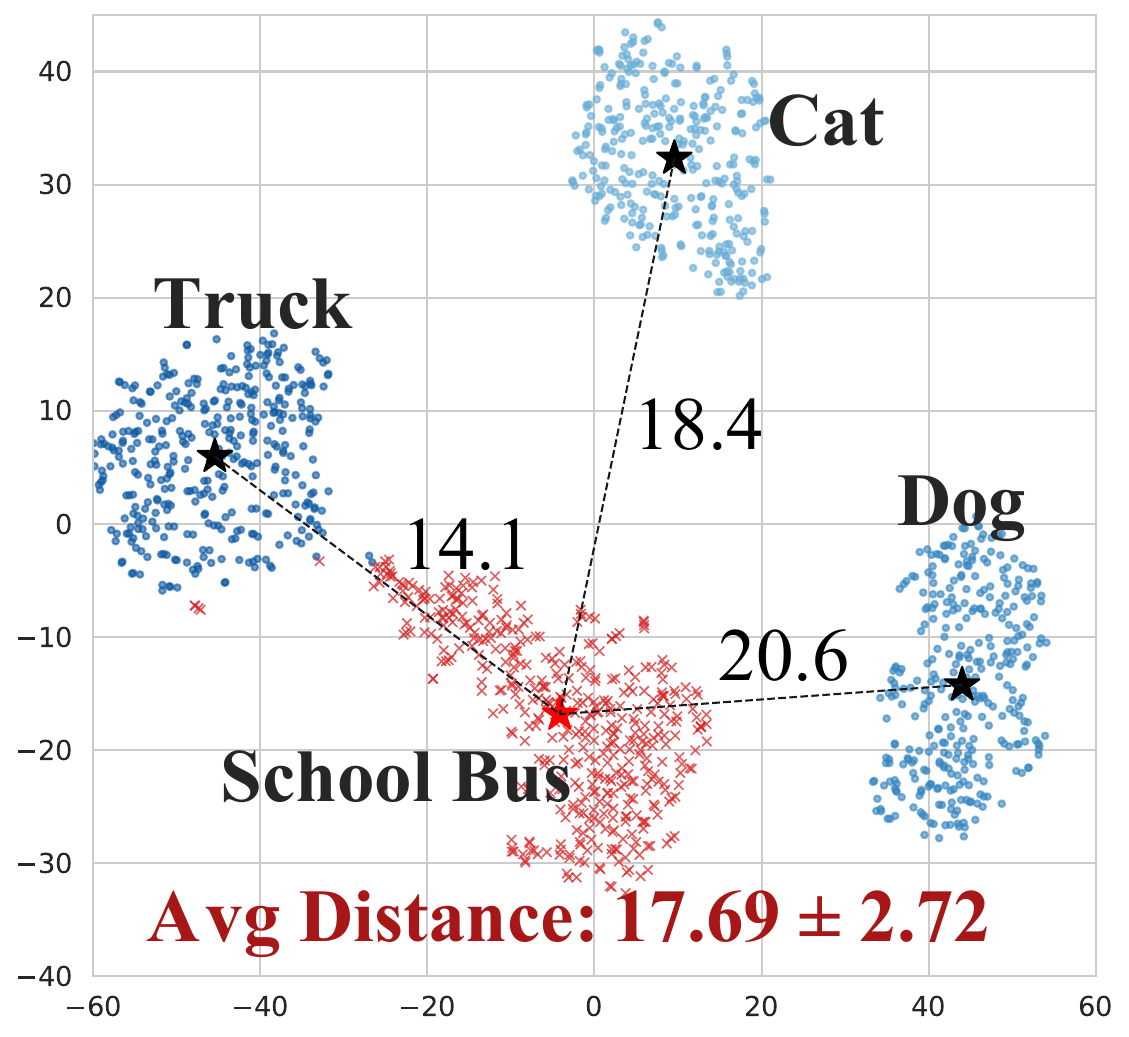}\\
{\small (a). $t$-SNE of baseline (left) and ours (right).}
\vspace{5pt}
\end{minipage}
\begin{minipage}[b]{\linewidth}\centering
\includegraphics[width=0.32\linewidth]{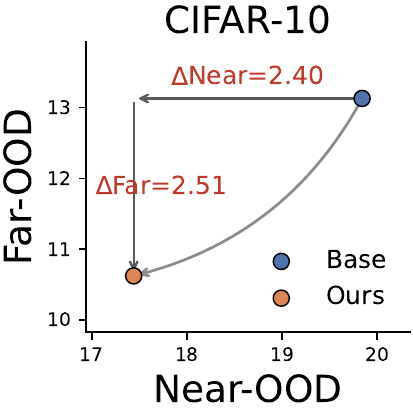}
\includegraphics[width=0.32\linewidth]{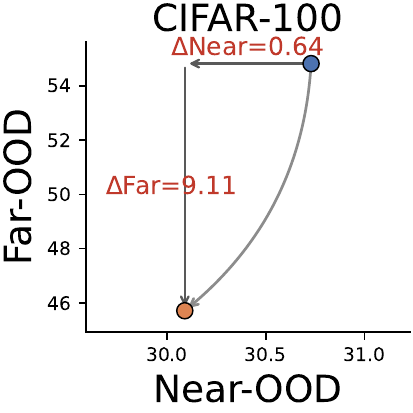}
\includegraphics[width=0.32\linewidth]{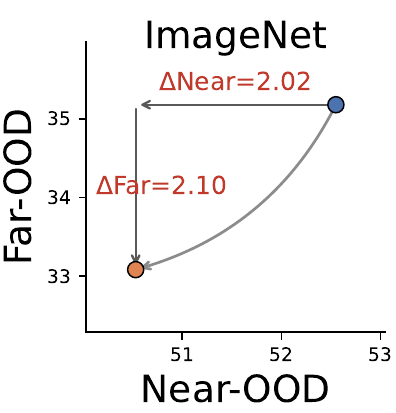} \\
{\small (b). OOD detection performance (FPR95).}
\end{minipage}
\caption{
(a) $t$-SNE visualization comparing the baseline (uniform labels) and our method (preserving relations) for the OOD category and ID categories. Compared with the baseline, our method exhibits a larger separation margin between ID and OOD samples. (b) The \textit{over-softening effect} is consistently observed on CIFAR-10, CIFAR-100, and ImageNet, where our method leads to improved OOD detection performance.}
\label{fig:intro}
\end{figure}

Pioneering works like MSP \cite{MSP:conf/iclr/HendrycksG17} and Energy~\cite{Energy:conf/nips/LiuWOL20} are designed to develop score functions that can effectively distinguish OOD samples from ID data. Nevertheless, the fixed parameters of these models impose a significant constraint on their adaptability and overall performance. To tackle this issue, a variety of training-time regularization methods~\cite{DBLP:conf/aaai/XuY26,LogitNorm:conf/icml/WeiXCF0L22,SSOD:conf/iclr/Pei24} such as LogitNorm~\cite{LogitNorm:conf/icml/WeiXCF0L22} have been introduced. Training-time regularization methods aim to modify learning strategies so as to enhance the discriminative capacity of learned representations for improved OOD detection. Meanwhile, with the aid of additional auxiliary data, outlier exposure~(OE)-based methods~\cite{OE:conf/iclr/HendrycksMD19,DOE:conf/iclr/WangY0DKLH023,DOS:conf/iclr/JiangCCWW24}, e.g. OE~\cite{OE:conf/iclr/HendrycksMD19} and DOE~\cite{DOE:conf/iclr/WangY0DKLH023}, facilitate the learning of more robust and generalizable decision boundaries. Among existing OOD detection methods, OE-based methods stand out as the most effective approach and consistently achieve superior performance as it leverages surrogate OOD samples during training to help the model distinguish between ID and OOD distributions.

Despite the effectiveness of OE-based methods, they inevitably disregard the relations of OOD samples with respect to different ID categories, which degrades OOD detection performance. We use a toy experiment to illustrate this issue. As shown  in \cref{fig:intro}(a), where ID categories comprise \textit{cat}, \textit{dog}, and \textit{truck}, while the OOD category is the \textit{school bus}. To enlarge the margin between ID and OOD samples, OE-based methods assign uniform labels to OOD samples~\cite{OE:conf/iclr/HendrycksMD19, DAL:conf/nips/WangFZLLH23}, which inevitably gives rise to an \textit{over-softening effect}, resulting in nearly uniform distances between the model predictions of OOD samples and those of different ID samples. In particular, OOD samples maintain a smaller distance to semantically irrelevant categories~(e.g., \textit{Cat}) than to the most relevant one~(e.g., \textit{Truck}), as shown in \cref{fig:intro}(a) (left). Meanwhile, as shown in \cref{thm:delta_margin_upper}, disregarding such relations leads to a suboptimal separation bound, thereby degrading OOD detection performance. Intuitively, OOD samples often contain semantic information with certain ID categories~\cite{ood_semantic_semi_supervised:conf/cvpr/WangQLSZC23, ood_useful:conf/eccv/WallinSKH24, ood_cvpr:conf/cvpr/ChenD25}. A desirable model behavior is therefore to maintain relatively lower distance to the most semantically similar ID category compared with others~\cite{LogitNorm:conf/icml/WeiXCF0L22}. \cref{thm:T_mitigate} further shows that preserving relations mitigates the over-softening effect, as shown in \cref{fig:intro}(a) (right) and improves OOD detection, which is also observed empirically in \cref{fig:intro}(b).

Building on this insight, we propose \underline{A}daptive Confidence \underline{OE} (AOE), a simple yet effective method that adaptively refines outlier labels via temperature scaling. Specifically, AOE employs a learnable temperature to smooth the model predictions for OOD samples and uses the resulting soft targets to supervise OOD training, thereby preserving the relations between OOD samples and ID categories; meanwhile, optimizing the temperature drives these smoothed targets toward a high-entropy distribution. To ensure robust convergence, AOE supports both joint optimization of model parameters and temperatures, as well as an alternating optimization strategy that decouples their respective update dynamics. Extensive experiments across diverse benchmarks demonstrate that AOE consistently outperforms existing state-of-the-art methods and can be seamlessly integrated with them.

The contributions of this work are summarized as follows:

\begin{itemize}
    \item We provide a theoretical analysis showing that the commonly used uniform labeling strategy in OE induces an \emph{over-softening effect}, which overlooks the intrinsic relationships between OOD samples and ID categories, leading to a suboptimal separation margin. We further demonstrate that introducing a temperature parameter $T$ can effectively alleviate this issue.

    \item Motivated by this analysis, we propose AOE, a simple yet effective method that recalibrates outlier labels via temperature scaling. By leveraging temperature-scaled model predictions as soft targets, AOE preserves semantic relationships with ID classes while promoting high-entropy distributions, thereby improving margin separation.

    \item We conduct extensive experiments on multiple OOD detection benchmarks, demonstrating that AOE consistently improvements over corresponding OE-based methods and achieves robust gains across diverse datasets and model architectures.
\end{itemize}
\section{Related Work}

\subsection{OOD Detection}
OOD detection plays a vital role in open-world machine learning by accurately classifying ID samples while identifying OOD samples~\cite{MSP:conf/iclr/HendrycksG17, domain_adapation_noise}. 
Early post-hoc OOD detection methods~\cite{GradNorm:conf/nips/HuangGL21, KNN-based:conf/icml/SunM0L22} primarily focus on designing discriminative scoring functions to separate OOD samples at inference time. In contrast, training-time regularization methods \cite{LogitNorm:conf/icml/WeiXCF0L22, T2FNorm:conf/cvpr/RegmiPDGSB22, SNN:conf/aaai/GhosalSL24, UM:conf/icml/ZhuLYLX023, pskd:conf/icml/0074X25} aim to improve OOD detection by designing training strategies that encourage the learning of more discriminative representations.

Unlike the aforementioned methods, OE-based methods~\cite{OE:conf/iclr/HendrycksMD19} incorporate realistic outliers into OOD detection models to enhance the robustness of decision boundaries. OE encourages OOD samples to exhibit a more uniform response during training, thereby reducing model overconfidence~\cite{OE:conf/iclr/HendrycksMD19}. Building on this idea, MixOE~\cite{MixOE:conf/wacv/ZhangILCL23} leverages MixUp between ID and OOD samples to synthesize more diverse OOD samples. DOE~\cite{DOE:conf/iclr/WangY0DKLH023} further enhances robustness by generating hard OOD samples through model-driven transformations, enabling more reliable detection. DAL~\cite{DAL:conf/nips/WangFZLLH23} synthesizes OOD samples by searching for the worst case within a Wasserstein ball around the OOD distribution, thereby reducing distribution discrepancy. OCL~\cite{ocl:conf/aaai/MiaoP0LZ24} explicitly assigns OOD samples to an additional $(K\!+\!1)$-th outlier class. By introducing outliers during training, these methods significantly improve OOD detection performance. However, these methods typically supervise OOD samples with uniform labels, which inevitably leads to the \textit{over-softening effect}.

\subsection{Temperature Scaling}

Since the temperature coefficient controls the sharpness of the softmax output distribution, temperature scaling~\cite{TS:conf/icml/GuoPSW17,KD:journals/corr/HintonVD15, domain_scaling:conf/nips/YuB0J22} has attracted significant attention. As a post-processing calibration technique, temperature scaling is first popularized for model calibration~\cite{TS:conf/icml/GuoPSW17}. Temperature scaling~\cite{TS:conf/icml/GuoPSW17} proposes an effective method to rescale the logits using a shared temperature parameter. T-CIL \cite{t_cil:conf/cvpr/Hwang0W25} optimizes the temperature by minimizing the calibration loss on adversarially perturbed samples, where perturbations are guided by feature-space distances. Additionally, temperature scaling has been widely applied in knowledge distillation~\cite{KD:journals/corr/HintonVD15} and contrastive learning~\cite{SimCLR:conf/icml/ChenK0H20}. Inspired by these practices, we incorporate temperature scaling to recalibrate outlier labels.

\section{Methodology}

\revised{In this section, we first introduce the preliminaries of OOD detection. We then provide a theoretical analysis showing that uniform outlier labels cause the over-softening effect and lead to a suboptimal separation margin. Motivated by this, we propose Adaptive Confidence OE to recalibrate outlier labels via temperature scaling, followed by the optimization strategy and its extension to different OE-based methods. The frequently used notations are summarized in \cref{tab:notation}.
}

\begin{table}[t]
\centering
\caption{Frequently used notations and their mathematical meanings.}
\label{tab:notation}
\footnotesize
\begin{tabular}{@{}>{\centering\arraybackslash}p{0.27\columnwidth}p{0.67\columnwidth}@{}}
\hline
Notation & Meaning \\
\hline
$\mathcal{X}$ & Input space. \\
$\mathcal{Y}=\{y_1,\ldots,y_K\}$ & Label set of $K$ ID categories. \\
$\mathcal{D}_{\mathrm{ID}}$, $\mathcal{D}_{\mathrm{OOD}}$ & ID and OOD distributions. \\
$x_i$, $x_o$ & ID and OOD samples. \\
$f(\cdot;\theta)$ & Classification model with parameters $\theta$. \\
$s(\cdot)$ & Softmax function. \\
$S(x;\theta)$ & OOD confidence score. \\
$\Delta(\theta)$ & Separation margin between ID and OOD samples. \\
$U$ & Uniform distribution over ID categories. \\
$m(f(x;\theta))$ & Prediction difference between the top two logits. \\
$\mu_i$, $\mu_o$, $\nu_o^2$ & Statistics of ID/OOD prediction differences. \\
$T$ & Learnable temperature for outlier-label recalibration. \\
$q_T=s(f(x_o;\theta)/T)$ & Temperature-scaled outlier target. \\
\hline
\end{tabular}
\end{table}

\subsection{Preliminary}
Let $\mathcal{X}$ denote the input space and $\mathcal{Y} = \{y_1, \dots, y_K\}$ the label set of $K$ classes. The goal of OOD detection is to determine whether a sample $x \in \mathcal{X}$ originates from $\mathcal{D}_{\text{ID}}$ or $\mathcal{D}_{\text{OOD}}$. Given a model $f(\cdot;\theta)$ trained on ID and OOD samples, an OOD score function $S(x;\theta)$ assigns a confidence value to $x$. In general, the OOD detector $g_\lambda(x)$ is given by:
\begin{equation}
g_\lambda(x) =
\begin{cases}
\text{ID}, & \text{if } S(x;\theta) \ge \lambda, \\
\text{OOD}, & \text{otherwise.}
\end{cases}
\end{equation}
where $\lambda$ is a threshold typically chosen to correctly classify the majority of ID samples (e.g., $95\%$). A sample is thus regarded as ID only if its confidence exceeds the threshold.

\subsection{Theoretical Analysis in Outlier Exposure}
OE-based methods improve OOD detection by introducing auxiliary outliers into the standard ID training. During training, these outliers are explicitly supervised with uniform labels to maximize entropy, thereby regularizing the decision boundary in regions beyond the support of the ID distribution~\cite{OE:conf/iclr/HendrycksMD19}. Specifically, the training objective can be expressed as:
\begin{equation}
\min_{\theta}\;
\mathcal{L}_{\mathrm{ce}}\!\left(s(f(x_i;\theta))\right)
\;+\;
\alpha
\mathcal{L}_{\mathrm{KL}}
\!\left(
\mathcal{U}\,\big\|\, s(f(x_o;\theta))
\right),
\label{eq:oe_objective}
\end{equation}
where $\mathcal{U}=(\tfrac{1}{K},\dots,\tfrac{1}{K})$ denotes the uniform distribution, and $\alpha$ is a hyperparameter balancing the contributions of ID samples $x_i$ and OOD samples $x_o$. Here, $s(\cdot)$ represents the softmax function \cite{MSP:conf/iclr/HendrycksG17}.

To analyze the impact of outliers on OOD detection performance, we introduce a definition to characterize the separability margin of ID samples and OOD samples.

\begin{definition}[Separability Margin]
Separation between the $\mathcal{D}_{\mathrm{ID}}$ and the $\mathcal{D}_{\mathrm{OOD}}$ is quantified by the discrepancy between their induced confidence distributions. Specifically, the separation is defined as
\begin{equation}
\Delta(\theta)
\triangleq
\mathbb{E}_{x_i \sim \mathcal{D}_{\mathrm{ID}}}\!\left[ S(x_i;\theta) \right]
-
\mathbb{E}_{x_o \sim \mathcal{D}_{\mathrm{OOD}}}\!\left[ S(x_o;\theta) \right],
\end{equation}
\label{eq:distribution_separation}
where $S(\cdot;\theta)$ denotes an OOD score from the model.
\end{definition}

Based on \cref{eq:distribution_separation}, effective OOD detection requires maintaining a sufficiently large separation margin between the confidence score distributions of ID and OOD samples~\cite{DOE:conf/iclr/WangY0DKLH023, ocl:conf/aaai/MiaoP0LZ24}. However, the uniform supervision in Eq.~\eqref{eq:oe_objective} induces an inherent \textit{over-softening effect} on model predictions for OOD samples. Since the relations between OOD samples and ID categories are reflected in the differences among their prediction scores~\cite{DBLP:conf/acl/acl_relation_1, DBLP:conf/acl/acl_relation_2}, we can characterize the effect at the logits level as follows.

\begin{proposition}
\label{prop:oversoften}
Let $z = f(x;\theta)$ denote the logits of $x$ and $z^{u}$ denote the logits with uniform supervision. The prediction differences across ID categories are contracted:
\begin{equation}
\label{eq:oversoften}
\bigl| z_i^{u} - z_j^{u} \bigr|
\;\le\;
\bigl| z_i - z_j \bigr|,
\forall\, i,j \in \{1,\dots,K\},\ i \neq j.
\end{equation}
\end{proposition}
\cref{prop:oversoften} shows that uniform supervision systematically contracts prediction differences across ID categories, driving OOD predictions toward a uniform configuration.

\begin{theorem}
\label{thm:delta_margin_upper}
For any $x$, $k^\star(x)=\arg\max_k z_k$, and define the prediction differences \(m\!\left(f(x;\theta)\right)
\triangleq
f_{k^\star(x)}(x;\theta)-\max_{j\neq k^\star(x)} f_j(x;\theta).\)
Then the separation margin admits the following explicit upper bound:
\begin{equation}
\begin{aligned}
\Delta(\theta)
\le \;
\frac{1}{1+\exp\!\big(-\mu_{i}\big)}
-
\frac{1}{1+(K-1)\exp\!\big(-\mu_{o}\big)}
\\
\;
-
\frac{(K-1)\exp\!\big(-\mu_{o}\big)}
     {2\big(1+(K-1)\exp\!\big(-\mu_{o}\big)\big)^2}
\,\nu_{o}^2
\;+\;
\mathcal{R}.
\end{aligned}
\end{equation}
where $\exp(\cdot)$ denotes the exponential function, $\mathcal{R}$ collects higher-order remainder terms, and $\mu_{i}$ and $\mu_{o}$ denote the means of the corresponding prediction differences, respectively, with $\nu_{o}^2$ representing the variance.
\end{theorem}

\cref{thm:delta_margin_upper} reveals an explicit trade-off in the separation margin. Specifically, a moderate reduction of the prediction differences of OOD samples enlarges the separation margin via the first-order effect, whereas excessive contraction introduces a negative quadratic correction that diminishes the separation. This analysis highlights that effective OOD detection requires balanced regulation of prediction differences rather than their aggressive suppression.

\begin{theorem}
\label{thm:T_mitigate}
Consider the scaled labels $q_T = s(f(x_o;\theta)/T)$ with $T>1$ for outlier $x_o$. Let $z^{t}$ denote the logits after training with scaled supervision. Then the induced update of the differences satisfies
\begin{equation}
\label{eq:margin_update_T}
\begin{aligned}
m(z^{t})
=\;
m(z)
-\eta\Bigg[
\frac{e^{m(z)}-1}{e^{m(z)}+1+\sum_{k\neq a,b}\exp(z_k-z_b)}
\\
\qquad\;
-
\frac{e^{m(z)/T}-1}{e^{m(z)/T}+1+\sum_{k\neq a,b}\exp\!\big((z_k-z_b)/T\big)}
\Bigg].
\end{aligned}
\end{equation}
where $a=\arg\max_k z_k$ and $b=\arg\max_{j\neq a} z_j$.
\end{theorem}

\begin{corollary}
\label{cor:finite_T}
Let $\Delta m_T \triangleq m(z)-m(z^{t})$ denote the contraction. Under the conditions of \cref{thm:T_mitigate}, the contraction induced by temperature scaling at an optimal finite temperature $T^\star\in(1,\infty)$ is strictly smaller than that induced by uniform outlier exposure as $T\to\infty$.
\end{corollary}

\cref{cor:finite_T} indicates that OE-based methods, corresponding to the extreme case $T\to\infty$, induces the strongest contraction. In contrast, a properly chosen finite temperature can smooth OOD predictions while preserving stronger relations. This observation motivates an adaptive mechanism for selecting the temperature to recalibrate outlier labels. Detailed proofs are provided in Appendix. Additionally, Appendix presents a theoretical analysis from the perspective of generalization bounds, substantiating the improved OOD detection capability.

\begin{figure}[t]
    \centering
    \includegraphics[width=0.6\linewidth]{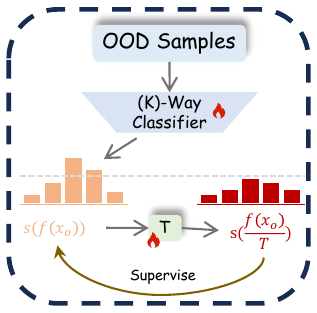}
    \caption{Illustration of model training with AOE. During training, the model leverages a learned temperature to smooth its outputs on OOD samples and uses the resulting softened predictions as outlier labels. The model parameters and the temperature are optimized to achieve a self-evolving calibration of the decision boundary.}
    \label{fig:framework}
\end{figure}

\subsection{Adaptive Confidence Outlier Exposure}
\label{sec:algo}
Based on the above analysis, there exists an optimal temperature \( T^* \) that balances the trade-off between smoothing OOD predictions and preserving the prediction differences across ID categories. However, this optimal temperature cannot be computed explicitly during training. To this end, we treat the temperature coefficient as a learnable parameter within the OOD detection model. \cref{fig:framework} illustrates the training framework of our proposal. During inference, the learned temperature $T$ is not used to calibrate the model outputs.

\begin{algorithm}[t]
\caption{Training procedure of AOE}
\label{alg:aoe}
\begin{algorithmic}[1]
\Require ID dataset $\mathcal{D}_{\mathrm{ID}}$, OOD dataset $\mathcal{D}_{\mathrm{OOD}}$, model $f(\cdot;\theta)$, temperature $T$, trade-off parameter $\alpha$
\Ensure Trained model parameter $\theta$

\For{each training iteration}
    \State Sample an ID mini-batch $(x_i,y_i)$ from $\mathcal{D}_{\mathrm{ID}}$
    \State Sample an OOD mini-batch $x_o$ from $\mathcal{D}_{\mathrm{OOD}}$
    
    \State Compute ID prediction $p_i = s(f(x_i;\theta))$
    
    \State Compute OOD prediction $p_o = s(f(x_o;\theta))$
    
    \State Compute temperature-scaled target $q_T = s(f(x_o;\theta)/T)$
    
    \State Update $\theta$ and $T$ by \cref{eq:optimization_joint} or \cref{eq:optimization_alte}
\EndFor

\State \Return $\theta$
\end{algorithmic}
\end{algorithm}

\noindent\textbf{Optimization Objective.} Temperature $T$ is optimized to regulate the model’s prediction on OOD samples, encouraging the temperature-scaled predictions to align with a uniform distribution while aligning the model outputs with these softened targets so as to preserve the relation of OOD samples with respect to ID categories. Accordingly, the optimization objective is formulated as:
\begin{equation}
\label{eq:ood_objective}
\begin{aligned}
\min_{T}\;\;
&\mathcal{L}_{\mathrm{align}}
\big(s(f(x_o;\theta)/T), \mathcal{U} \big)
\\
&+
\mathcal{L}_{\mathrm{align}}
\big(s(f(x_o;\theta)), s(f(x_o;\theta)/T) \big),
\end{aligned}
\end{equation}
where $\mathcal{L}_{\mathrm{align}}(\cdot)$ denotes a distribution alignment loss, such as the Kullback--Leibler divergence \cite{kl_diver:journals/tcyb/JiZYWZG22}. In this formulation, the optimal temperature $T^*$ is estimated from the current model predictions, while the learned temperature simultaneously influences the model outputs through temperature scaling \cite{domain_scaling:conf/nips/YuB0J22, t_cil:conf/cvpr/Hwang0W25}. Model parameters are optimized using available supervisory signals, where ID samples are supervised by ground-truth labels \cite{celoss:conf/icml/0003HNH24} and OOD samples are supervised by the smoothed predictions. 

\noindent\textbf{Optimization Strategy.}
To determine the optimal model parameters $\theta$ and temperature $T$, we consider two optimization paradigms, namely \textit{Joint Training} and \textit{Alternating Training}.

\textit{Joint Training.} This paradigm optimizes $\theta$ and $T$ simultaneously within a single computational graph. Both variables are updated based on the combined objective:
\begin{equation}
\begin{aligned}
\min_{\theta, T}\;\;
& \mathcal{L}_{\mathrm{ce}}
\big(s(f(x_i;\theta)), y \big) + \alpha \Big(
\mathcal{L}_{\mathrm{align}}
\big(s(f(x_o;\theta)/T), \mathcal{U} \big) \\
& \qquad\;\; + \mathcal{L}_{\mathrm{align}}
\big(s(f(x_o;\theta)), s(f(x_o;\theta)/T) \big)
\Big)
\end{aligned}
\label{eq:optimization_joint}
\end{equation}
where $\mathcal{L}_{\mathrm{ce}}(\cdot)$ denotes the standard cross-entropy loss \cite{celoss:conf/icml/0003HNH24}, and $\alpha$ balances the contributions from ID and OOD samples \cite{OE:conf/iclr/HendrycksMD19}. 

\textit{Alternating Training.} This paradigm decouples the optimization by iteratively updating $T$ and $\theta$. Specifically, $T^*$ is first updated to satisfy the alignment constraint, and is subsequently treated as a constant to supervise $\theta$:
\begin{equation}
\begin{aligned}
\min_{\theta}\;\;
& \mathcal{L}_{\mathrm{ce}}
\big(s(f(x_i;\theta)), y \big)+\alpha \;\mathcal{L}_{\mathrm{align}}
\big(s(f(x_o;\theta)), s(f(x_o;\theta)/T^*) \big) \\
\text{s.t.}\;\;
& T^* \in \arg\min_{T}
\mathcal{L}_{\mathrm{align}}
\big(s(f(x_o;\theta)/T), \mathcal{U} \big).
\end{aligned}
\label{eq:optimization_alte}
\end{equation}

In practice, Eq.~\eqref{eq:optimization_alte} is solved by alternating minimization \cite{bilevel:journals/eor/ParvasiMTYL26, bilevel_2:journals/pr/DengLPSX26}. At iteration \(t\), the temperature update is given by
\begin{equation}
T^{(t+1)} = T^{(t)} - \eta_T
\nabla_T \mathcal{L}_{\mathrm{align}}
\big(s(f(x_o;\theta^{(t)})/T^{(t)}), \mathcal{U} \big),
\label{eq:T_update}
\end{equation}
where \(\eta_T\) denotes the learning rate for the temperature update. Given \(T^{(t+1)}\), the model parameters are updated by
\begin{equation}
\begin{aligned}
\theta^{(t+1)}
= \theta^{(t)} - \eta_\theta
\nabla_\theta \Big(
 \mathcal{L}_{\mathrm{ce}}
\big(s(f(x_i;\theta^{(t)})), y_i \big) \\
 + \mathcal{L}_{\mathrm{align}}
\big(s(f(x_o;\theta^{(t)})),
     s(f(x_o;\theta^{(t)})/T^{(t+1)}) \big)
\Big),
\end{aligned}
\label{eq:theta_update}
\end{equation}
where \(\eta_\theta\) denotes the learning rate for updating \(\theta\). In the parameter update step, \(T^{(t+1)}\) is treated as a constant, yielding a first-order approximation to the optimization objective. The overall algorithm of our model is outlined in \cref{alg:aoe}.

\begin{table*}[t]
\centering
\caption{Comparison on ImageNet benchmarks. All values are percentages, and OOD detection results are averaged over multiple OOD test datasets. The best results are in \textbf{bold}, with the second-best \underline{underlined}. Detailed results for each OOD dataset are provided in Appendix. $\uparrow$ indicates that larger values are better, while $\downarrow$ indicates that smaller values are better. \textit{AOE-At} denotes alternating training, while \textit{AOE-Jt} denotes joint training.}
\label{tab:imagenet_ood}
\begin{tabular}{lccccccc}
\toprule
Method 
& \multicolumn{2}{c}{Near-OOD}
& \multicolumn{2}{c}{Far-OOD}
& \multicolumn{2}{c}{Average}
& ID ACC$\uparrow$ \\
\cmidrule(lr){2-3} \cmidrule(lr){4-5} \cmidrule(lr){6-7}
 & FPR95$\downarrow$ & AUROC$\uparrow$
 & FPR95$\downarrow$ & AUROC$\uparrow$
 & FPR95$\downarrow$ & AUROC$\uparrow$
 &  \\
\midrule
\multicolumn{8}{c}{\textit{Post-hoc}} \\
MSP        & 54.82\tiny$\pm$0.35 & 83.34\tiny$\pm$0.06 & 35.43\tiny$\pm$0.38 & 90.13\tiny$\pm$0.09 & 45.13\tiny$\pm$0.32 & 86.74\tiny$\pm$0.07 & 86.37\tiny$\pm$0.08 \\
Energy     & 60.24\tiny$\pm$0.57 & 82.50\tiny$\pm$0.05 & 34.86\tiny$\pm$1.30 & 90.86\tiny$\pm$0.21 & 47.55\tiny$\pm$0.76 & 86.68\tiny$\pm$0.12 & 86.37\tiny$\pm$0.08 \\
\midrule
\multicolumn{8}{c}{\textit{Training-time regularization}} \\
LogitNorm  & 57.80\tiny$\pm$1.22 & 82.21\tiny$\pm$0.52 & 25.31\tiny$\pm$0.20 & 93.31\tiny$\pm$0.15 & 41.56\tiny$\pm$0.71 & 87.76\tiny$\pm$0.33 & 86.18\tiny$\pm$0.60 \\
T2FNorm    & 55.65\tiny$\pm$0.20 & 82.72\tiny$\pm$0.05 & 25.25\tiny$\pm$0.48 & 93.38\tiny$\pm$0.11 & 40.45\tiny$\pm$0.33 & 88.06\tiny$\pm$0.07 & 86.52\tiny$\pm$0.21 \\
UM         & 60.23\tiny$\pm$1.13 & 81.79\tiny$\pm$0.38 & 32.46\tiny$\pm$1.30 & 91.68\tiny$\pm$0.37 & 46.34\tiny$\pm$0.81 & 86.74\tiny$\pm$0.28 & 85.01\tiny$\pm$0.31 \\
UMAP       & 60.81\tiny$\pm$0.84 & 81.08\tiny$\pm$0.39 & 32.47\tiny$\pm$0.67 & 91.62\tiny$\pm$0.29 & 46.64\tiny$\pm$0.74 & 86.35\tiny$\pm$0.32 & 86.37\tiny$\pm$0.08 \\
PSKD       & 57.12\tiny$\pm$0.63 & 82.84\tiny$\pm$0.32 & 31.64\tiny$\pm$0.87 & 91.39\tiny$\pm$0.16 & 44.38\tiny$\pm$0.22 & 87.12\tiny$\pm$0.15 & 86.79\tiny$\pm$0.25 \\
\midrule
\multicolumn{8}{c}{\textit{Outlier Exposure}} \\
MixOE          & 86.79\tiny$\pm$0.25 & 82.57\tiny$\pm$0.23 & 40.12\tiny$\pm$0.66 & 88.39\tiny$\pm$0.02 & 49.03\tiny$\pm$0.44 & 85.49\tiny$\pm$0.12 & 85.73\tiny$\pm$0.09 \\
DOE            & 54.14\tiny$\pm$0.51 & 83.23\tiny$\pm$0.60 & 37.60\tiny$\pm$2.95 & 88.24\tiny$\pm$1.45 & 45.87\tiny$\pm$1.72 & 85.73\tiny$\pm$1.00 & 79.87\tiny$\pm$3.12 \\
DAL            & 51.85\tiny$\pm$0.40 & 85.17\tiny$\pm$0.05 & 35.76\tiny$\pm$0.14 & 88.55\tiny$\pm$0.05 & 43.81\tiny$\pm$0.27 & 86.86\tiny$\pm$0.05 & 86.18\tiny$\pm$0.29 \\
OE            & 52.55\tiny$\pm$0.51 & 84.51\tiny$\pm$0.21 & 35.18\tiny$\pm$0.63 & 88.21\tiny$\pm$0.19 & 43.86\tiny$\pm$0.49 & 86.36\tiny$\pm$0.20 & 85.78\tiny$\pm$0.12 \\
OCL             & 51.64\tiny$\pm$0.52 & 84.21\tiny$\pm$0.42 & 34.30\tiny$\pm$1.53 & 89.78\tiny$\pm$0.49 & 42.97\tiny$\pm$1.02 & 86.99\tiny$\pm$0.45 & 86.27\tiny$\pm$0.09 \\
\midrule
\textit{AOE-At}       & \bf50.53\tiny$\pm$0.05 & \bf85.61\tiny$\pm$0.09 & \underline{34.61\tiny$\pm$0.05} & \underline{88.95\tiny$\pm$0.14} & \underline{42.57\tiny$\pm$0.33} & \underline{87.28\tiny$\pm$0.12} & \underline{86.51\tiny$\pm$0.05} \\
\textit{AOE-Jt}       & \underline{50.83\tiny$\pm$0.14} & \underline{85.57\tiny$\pm$0.07} & \bf33.08\tiny$\pm$0.53 & \bf89.24\tiny$\pm$0.12 & \bf41.95\tiny$\pm$0.33 & \bf87.41\tiny$\pm$0.09 & \bf86.49\tiny$\pm$0.20 \\
\bottomrule
\end{tabular}
\end{table*}

\begin{table*}[!t]
\centering
\caption{OOD detection performance of multiple OE-based methods and their AOE-enhanced variants on CIFAR-10, CIFAR-100, and ImageNet-200. AOE generally improves the corresponding OE-based baselines in averaged near-OOD and far-OOD evaluations, with gains that are more pronounced in several far-OOD settings. These results demonstrate the effectiveness and general applicability of AOE as a plug-and-play enhancement for existing OE strategies.}
\label{tab:plug-play_oe}
\resizebox{\textwidth}{!}{
\begin{tabular}{lllcccccccc}
\toprule
ID Dataset & Type & Metric & OE & +AOE & MixOE & +AOE & DOE & +AOE & DAL & +AOE \\
\midrule
\multirow{4}{*}{CIFAR-10} 
& \multirow{2}{*}{Near-OOD} & FPR95 $\downarrow$ &19.84\tiny$\pm$0.95 &\bf17.44\tiny$\pm$0.05  &51.45\tiny$\pm$7.78 &\bf47.11\tiny$\pm$5.05  &20.39\tiny$\pm$0.15 &\bf18.96\tiny$\pm$0.78 &20.91\tiny$\pm$0.71 & \bf19.27\tiny$\pm$0.79 \\
&                           & AUROC $\uparrow$   &94.82\tiny$\pm$0.21 &\bf95.64\tiny$\pm$0.03  &88.73\tiny$\pm$0.82 &\bf89.07\tiny$\pm$0.32  &94.84\tiny$\pm$0.07 &\bf95.17\tiny$\pm$0.04 &94.42\tiny$\pm$0.30 & \bf94.69\tiny$\pm$0.22   \\
& \multirow{2}{*}{Far-OOD}  & FPR95 $\downarrow$ &13.13\tiny$\pm$0.53 &\bf10.62\tiny$\pm$1.50  &33.84\tiny$\pm$4.77 &\bf26.14\tiny$\pm$1.80  &15.59\tiny$\pm$1.47 &\bf13.46\tiny$\pm$1.21 & 21.40\tiny$\pm$2.56 & \bf20.32\tiny$\pm$0.27  \\
&                           & AUROC $\uparrow$   &96.00\tiny$\pm$0.13 &\bf97.04\tiny$\pm$0.47  &91.93\tiny$\pm$0.69 &\bf93.61\tiny$\pm$0.22  &94.67\tiny$\pm$0.69 &\bf96.16\tiny$\pm$0.10 &91.92\tiny$\pm$1.36 & \bf92.12\tiny$\pm$0.23  \\
\midrule
\multirow{4}{*}{CIFAR-100} 
& \multirow{2}{*}{Near-OOD} & FPR95 $\downarrow$ &30.73\tiny$\pm$0.11 &\bf30.09\tiny$\pm$0.21 &55.22\tiny$\pm$0.49  &\bf54.70\tiny$\pm$0.62  &37.84\tiny$\pm$1.05 &\bf32.21\tiny$\pm$2.68 & 35.28\tiny$\pm$1.13 & \bf32.16\tiny$\pm$0.85  \\
&                           & AUROC $\uparrow$   &88.30\tiny$\pm$0.10 &\bf88.32\tiny$\pm$0.10 &80.95\tiny$\pm$0.20  &\bf81.00\tiny$\pm$0.31  &86.61\tiny$\pm$0.29 &\bf87.57\tiny$\pm$1.29 & 85.97\tiny$\pm$0.07 & \bf87.61\tiny$\pm$0.33  \\
&  \multirow{2}{*}{Far-OOD} & FPR95 $\downarrow$ &54.82\tiny$\pm$2.79 &\bf46.12\tiny$\pm$2.26 &63.88\tiny$\pm$2.48  &\bf60.33\tiny$\pm$3.16  &45.38\tiny$\pm$0.52 &\bf43.64\tiny$\pm$3.87 & 47.33\tiny$\pm$1.28 & \bf41.59\tiny$\pm$1.59  \\
&                           & AUROC $\uparrow$   &81.41\tiny$\pm$1.49 &\bf86.04\tiny$\pm$0.38 &76.40\tiny$\pm$1.44  &\bf78.16\tiny$\pm$1.05  &84.30\tiny$\pm$0.81 &\bf86.56\tiny$\pm$3.88 & 82.42\tiny$\pm$0.26 & \bf85.80\tiny$\pm$0.78  \\
\midrule
\multirow{4}{*}{ImageNet-200} 
& \multirow{2}{*}{Near-OOD} & FPR95 $\downarrow$ &52.55\tiny$\pm$0.51 &\bf50.83\tiny$\pm$0.14 &57.95\tiny$\pm$0.23  &\bf54.82\tiny$\pm$3.47  &54.14\tiny$\pm$0.51 &\bf50.34\tiny$\pm$0.11 & 51.85\tiny$\pm$0.40 &\bf50.78\tiny$\pm$0.36  \\
&                           & AUROC $\uparrow$   &84.51\tiny$\pm$0.21 &\bf85.57\tiny$\pm$0.07 &82.57\tiny$\pm$0.23  &\bf83.83\tiny$\pm$1.34  &83.23\tiny$\pm$0.60 &\bf85.60\tiny$\pm$0.01 & 85.17\tiny$\pm$0.05 &\bf86.20\tiny$\pm$0.06  \\
&  \multirow{2}{*}{Far-OOD} & FPR95 $\downarrow$ &35.18\tiny$\pm$0.63 &\bf33.08\tiny$\pm$0.53 &40.12\tiny$\pm$0.66  &\bf37.98\tiny$\pm$3.23  &37.60\tiny$\pm$2.95 &\bf34.68\tiny$\pm$0.02 & 35.76\tiny$\pm$0.14 &\bf34.60\tiny$\pm$0.10  \\
&                           & AUROC $\uparrow$   &88.21\tiny$\pm$0.19 &\bf89.24\tiny$\pm$0.12 &88.39\tiny$\pm$0.02  &\bf88.67\tiny$\pm$0.18  &88.24\tiny$\pm$1.45 &\bf88.81\tiny$\pm$0.25 & 88.55\tiny$\pm$0.05 &\bf88.60\tiny$\pm$0.06  \\
\bottomrule
\end{tabular}
}
\end{table*}

\noindent\textbf{Extension to Different OE Strategies.} AOE seamlessly adapts to diverse OE strategies. For methods focused on generating synthetic OOD samples~\cite{DAL:conf/nips/WangFZLLH23}, we directly adopt the objective defined in Eq.~\eqref{eq:ood_objective}. Conversely, when leveraging auxiliary outlier classes~\cite{ ocl:conf/aaai/MiaoP0LZ24}, we align the predictive distribution over ID categories with a softened target distribution obtained via temperature scaling. Accordingly, Eq.~\eqref{eq:ood_objective} can be reformulated as follows:
\begin{equation}
\begin{aligned}
\min_{T}\;\;
&\mathcal{L}_{\mathrm{align}}
\big(s(f_{1:K}(x_o;\theta)/T), \mathcal{U} \big)
\\
&+
\mathcal{L}_{\mathrm{align}}
\big(s(f_{1:K}(x_o;\theta)), s(f_{1:K}(x_o;\theta)/T) \big),
\end{aligned}
\end{equation}
where $f_{1:K}(\cdot)$ denotes the logits corresponding to the ID categories and $K$ represents the number of ID categories.

\section{Experiments}
\subsection{Experimental Setup}
To ensure a fair comparison, we follow established protocols and assess our approach on the OpenOODv1.5 benchmark~\cite{openood:journals/corr/abs-2306-09301}. The evaluation spans both small-scale and large-scale datasets, encompassing near-OOD characterized by semantic shifts as well as far-OOD involving more pronounced covariate shifts.

\begin{table*}[t]
\centering
\caption{Comparison on CIFAR benchmarks. All values are percentages, and OOD detection results are averaged over multiple OOD test datasets. The best results are in \textbf{bold}, with the second-best \underline{underlined}. Detailed results for each OOD dataset are provided in Appendix. $\uparrow$ indicates that larger values are better, while $\downarrow$ indicates that smaller values are better. \textit{AOE-At} denotes alternating training, while \textit{AOE-Jt} denotes joint training.}
\label{tab:cifar_ood}
\resizebox{\textwidth}{!}{
\begin{tabular}{lcccccccccc}
\toprule
 & \multicolumn{5}{c}{\textbf{CIFAR-10}} & \multicolumn{5}{c}{\textbf{CIFAR-100}} \\
\cmidrule(lr){2-6} \cmidrule(lr){7-11}
Method
& \multicolumn{2}{c}{Near-OOD}
& \multicolumn{2}{c}{Far-OOD}
& ID ACC$\uparrow$
& \multicolumn{2}{c}{Near-OOD}
& \multicolumn{2}{c}{Far-OOD}
& ID ACC$\uparrow$ \\
 & FPR95$\downarrow$ & AUROC$\uparrow$
 & FPR95$\downarrow$ & AUROC$\uparrow$
 & 
 & FPR95$\downarrow$ & AUROC$\uparrow$
 & FPR95$\downarrow$ & AUROC$\uparrow$
 &  \\
\midrule
\multicolumn{11}{c}{\textit{Post-hoc}} \\
MSP        & 48.17\tiny$\pm$3.92 & 88.03\tiny$\pm$0.25 & 31.72\tiny$\pm$1.84 & 90.73\tiny$\pm$0.43 & \bf95.06\tiny$\pm$0.30 & 54.80\tiny$\pm$0.33 & 80.27\tiny$\pm$0.11 & 58.70\tiny$\pm$1.06 & 77.76\tiny$\pm$0.44 & \bf77.25\tiny$\pm$0.10 \\
Energy     & 61.34\tiny$\pm$4.63 & 87.58\tiny$\pm$0.46 & 41.69\tiny$\pm$5.32 & 91.21\tiny$\pm$0.92 & \underline{95.06\tiny$\pm$0.30} & 55.62\tiny$\pm$0.61 & 80.91\tiny$\pm$0.08 & 56.59\tiny$\pm$1.38 & 79.77\tiny$\pm$0.61 & \underline{77.25\tiny$\pm$0.10} \\
\midrule
\multicolumn{11}{c}{\textit{Training-time regularization}} \\
LogitNorm  & 29.34\tiny$\pm$0.81 & 92.33\tiny$\pm$0.08 & 13.81\tiny$\pm$0.20 & 96.74\tiny$\pm$0.06 & 94.30\tiny$\pm$0.25 & 62.89\tiny$\pm$0.57 & 78.47\tiny$\pm$0.31 & 53.61\tiny$\pm$3.45 & 81.53\tiny$\pm$1.26 & 76.34\tiny$\pm$0.17 \\
T2FNorm    & 26.47\tiny$\pm$0.35 & 92.79\tiny$\pm$0.13 & 12.75\tiny$\pm$0.73 & 96.98\tiny$\pm$0.23 & 94.69\tiny$\pm$0.07 & 58.47\tiny$\pm$1.35 & 79.84\tiny$\pm$0.40 & 51.25\tiny$\pm$2.52 & 82.73\tiny$\pm$1.01 & 76.43\tiny$\pm$0.13 \\
UM         & 33.12\tiny$\pm$0.47 & 90.60\tiny$\pm$0.37 & 22.95\tiny$\pm$1.65 & 93.67\tiny$\pm$0.68 & 92.33\tiny$\pm$0.41 & 65.86\tiny$\pm$2.06 & 77.14\tiny$\pm$0.62 & 51.90\tiny$\pm$2.04 & 81.63\tiny$\pm$1.70 & 72.21\tiny$\pm$0.43 \\
UMAP       & 33.01\tiny$\pm$0.06 & 91.00\tiny$\pm$0.07 & 21.70\tiny$\pm$1.57 & 94.20\tiny$\pm$0.36 & 95.06\tiny$\pm$0.30 & 59.71\tiny$\pm$0.65 & 79.49\tiny$\pm$0.23 & 52.11\tiny$\pm$2.36 & 81.62\tiny$\pm$1.37 & 77.25\tiny$\pm$0.10 \\
PSKD       & 31.67\tiny$\pm$0.78 & 91.71\tiny$\pm$0.16 & 20.48\tiny$\pm$1.30 & 94.56\tiny$\pm$0.29 & 95.14\tiny$\pm$0.08 & 54.83\tiny$\pm$2.79 & 81.45\tiny$\pm$0.44 & 51.56\tiny$\pm$2.39 & 82.40\tiny$\pm$0.52 & 77.44\tiny$\pm$0.09 \\
\midrule
\multicolumn{11}{c}{\textit{Outlier Exposure}} \\
MixOE          & 51.45\tiny$\pm$7.78 & 88.73\tiny$\pm$0.82 & 33.84\tiny$\pm$4.77 & 91.93\tiny$\pm$0.69 & 94.55\tiny$\pm$0.32 & 55.22\tiny$\pm$0.49 & 80.95\tiny$\pm$0.20 & 63.88\tiny$\pm$2.48 & 76.40\tiny$\pm$1.44 & 75.13\tiny$\pm$0.06 \\
DOE            & 20.39\tiny$\pm$0.15 & 94.84\tiny$\pm$0.07 & 15.59\tiny$\pm$1.47 & 94.67\tiny$\pm$0.69 & 94.32\tiny$\pm$0.19 & 37.84\tiny$\pm$1.05 & 86.61\tiny$\pm$0.29 & \bf45.38\tiny$\pm$0.52 & 84.30\tiny$\pm$0.81 & 75.69\tiny$\pm$0.26 \\
DAL            & 20.91\tiny$\pm$0.71 & 94.42\tiny$\pm$0.30 & 21.40\tiny$\pm$2.56 & 91.92\tiny$\pm$1.36 & 92.95\tiny$\pm$0.54 & 35.28\tiny$\pm$1.13 & 85.97\tiny$\pm$0.07 & 47.33\tiny$\pm$1.28 & 82.42\tiny$\pm$0.26 & 75.60\tiny$\pm$0.34 \\
OE          & 19.84\tiny$\pm$0.95 & 94.82\tiny$\pm$0.21 & 13.13\tiny$\pm$0.53 & 96.00\tiny$\pm$0.13 & 94.63\tiny$\pm$0.26 & 30.73\tiny$\pm$0.11 & 88.30\tiny$\pm$0.10 & 54.82\tiny$\pm$2.79 & 81.41\tiny$\pm$1.49 & 76.84\tiny$\pm$0.42 \\
OCL            & 22.64\tiny$\pm$0.88 & 94.84\tiny$\pm$0.24 & 14.44\tiny$\pm$1.21 & \underline{96.71\tiny$\pm$0.30} & 93.83\tiny$\pm$0.32 & \underline{30.49\tiny$\pm$0.47} & \bf88.91\tiny$\pm$0.16 & 51.10\tiny$\pm$3.12 & 83.67\tiny$\pm$1.26 & 75.96\tiny$\pm$0.20 \\
\midrule
\textit{AOE-At}      & \underline{18.37\tiny$\pm$0.57} & \underline{95.21\tiny$\pm$0.09} & \underline{11.88\tiny$\pm$0.15} & 96.25\tiny$\pm$0.42 & 94.82\tiny$\pm$0.19 & 30.56\tiny$\pm$0.14 & 88.24\tiny$\pm$0.07 & \underline{45.71\tiny$\pm$2.88} & \underline{85.97\tiny$\pm$1.06} & 76.21\tiny$\pm$0.01 \\
\textit{AOE-Jt}      & \bf17.44\tiny$\pm$0.05 & \bf95.64\tiny$\pm$0.03 & \bf10.62\tiny$\pm$1.50 & \bf97.04\tiny$\pm$0.47 & 94.79\tiny$\pm$0.15 & \bf30.09\tiny$\pm$0.21 & \underline{88.32\tiny$\pm$0.10} & 46.12\tiny$\pm$2.26 & \bf86.04\tiny$\pm$0.38 & 76.90\tiny$\pm$0.14 \\
\bottomrule
\end{tabular}}
\end{table*}

\noindent\textbf{Datasets:} Small-scale evaluations are conducted using CIFAR-10 and CIFAR-100 \cite{cifar:journals/corr/abs-1811-07270} as in-distribution datasets. The near-OOD datasets include CIFAR-100/10 and Tiny ImageNet \cite{tiny-imagenet:journals/corr/abs-1904-10429}, while the far-OOD datasets comprise MNIST \cite{mnist:journals/spm/Deng12}, SVHN \cite{SVHN:conf/icit2/PradhanTMG24}, Textures \cite{textures:conf/3dpvt/GoolVKTZ02}, and Places365 \cite{zhou2017places}. In addition, Tiny ImageNet-597 \cite{tiny-imagenet:journals/corr/abs-1904-10429}, with no category overlap with CIFAR-10/100 or the OOD datasets, is employed as a realistic outlier set following the OpenOOD evaluation protocol. For large-scale experiments, ImageNet-200, consisting of 200 categories sampled from ImageNet-1K \cite{imagenet:conf/cvpr/DengDSLL009}, serves as the in-distribution dataset. The near-OOD evaluation utilizes SSB-hard \cite{SSB:conf/iclr/Vaze0VZ22} and NINCO \cite{NINCO}, whereas the far-OOD evaluation considers iNaturalist \cite{iNaturalist:conf/cvpr/HornASCSSAPB18}, Textures \cite{textures:conf/3dpvt/GoolVKTZ02}, and OpenImage-O \cite{vim:conf/cvpr/Wang0F022}. Furthermore, the remaining 800 classes from ImageNet-1K are used as realistic outliers.

\noindent\textbf{Implementation Details:} To ensure fair evaluation, we follow the standardized OOD detection protocol provided by the OpenOOD benchmark~\cite{openood:journals/corr/abs-2306-09301}. ResNet-18~\cite{resnet18:conf/cvpr/HeZRS16} is adopted as the primary backbone. All models are trained for 100 epochs using stochastic gradient descent with an initial learning rate of 0.1, cosine annealing scheduling~\cite{rate_scle:conf/iclr/LoshchilovH17}, a momentum coefficient of 0.9, and a weight decay of \(5 \times 10^{-4}\). The batch size is set to 128 for CIFAR-10 and -100, and 256 for ImageNet-200. During training, the temperature \(T\) is constrained within the range \([1.0, 10]\), and values outside this interval are clipped accordingly.

\begin{figure*}[!t]
    \centering
    \begin{minipage}{0.32\linewidth}
        \centering
        \includegraphics[width=\linewidth]{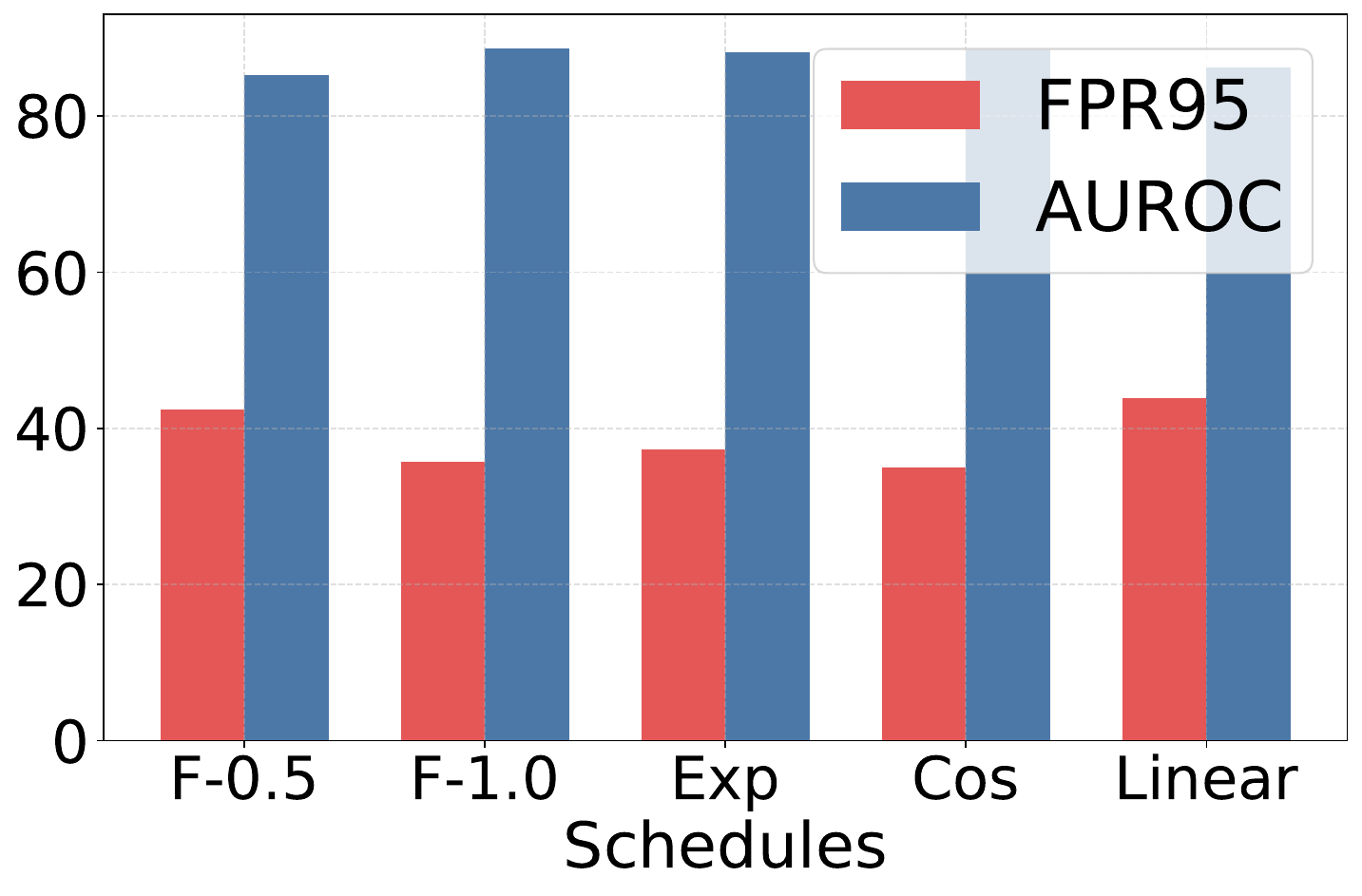}
        (a) Adjustment strategies of $\alpha$
    \end{minipage}
    \begin{minipage}{0.32\linewidth}
        \centering
        \includegraphics[width=\linewidth]{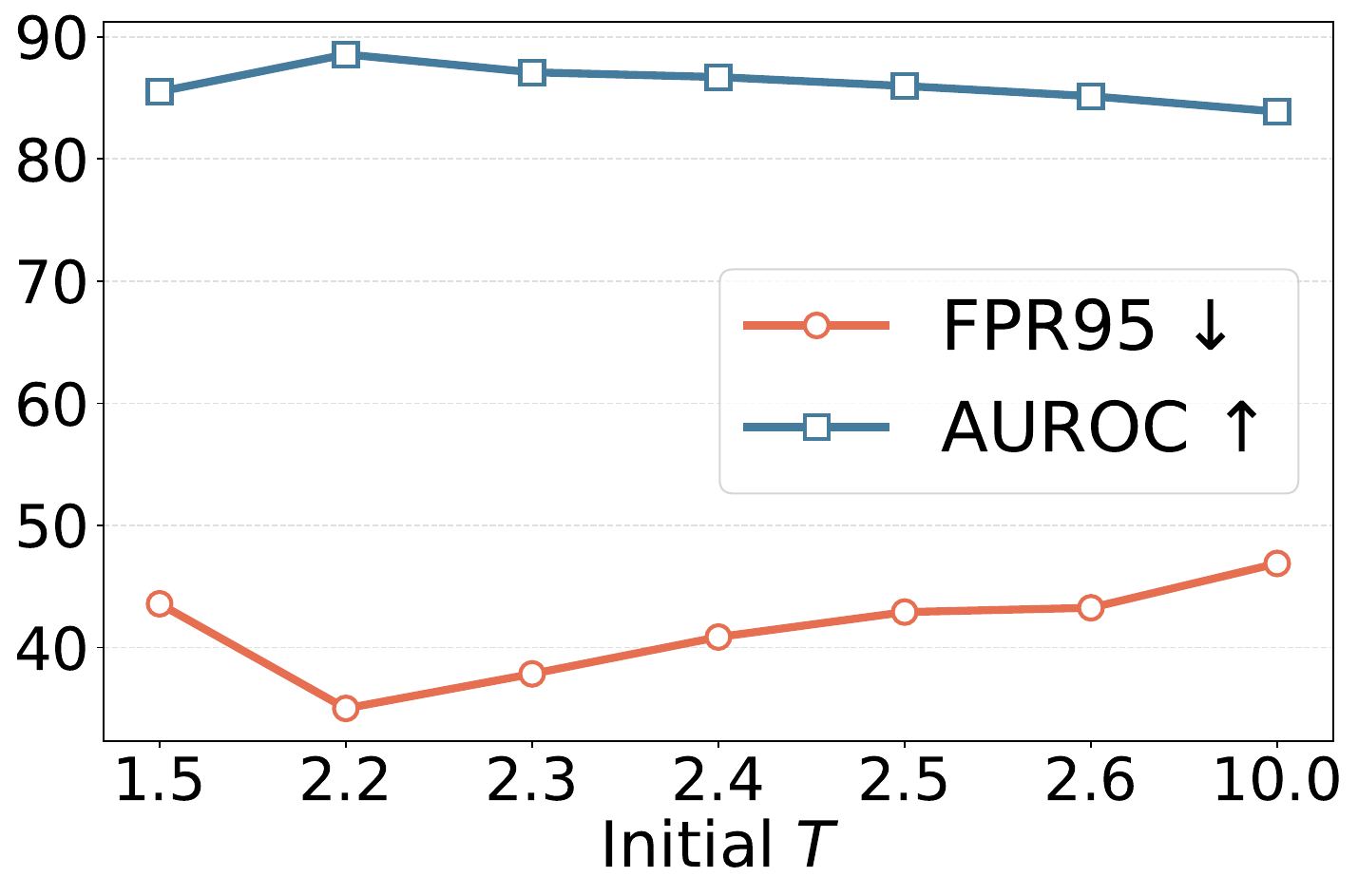}
        (b) Temperature initialization
    \end{minipage}
    \begin{minipage}{0.32\linewidth}
        \centering
        \includegraphics[width=\linewidth]{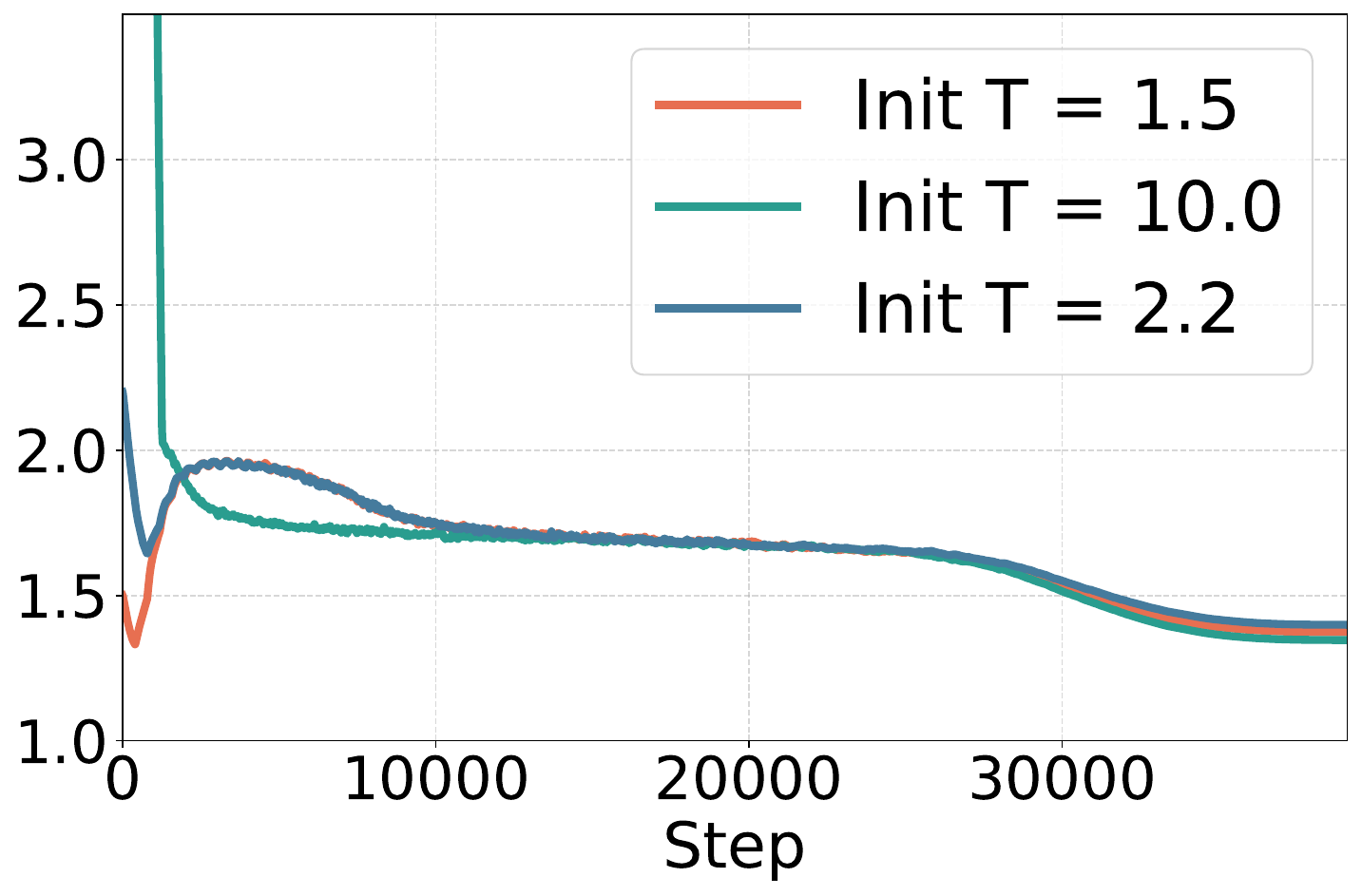} 
        (c) Evolution of $T$
    \end{minipage}
    \caption{Ablation study on key design choices of AOE. CIFAR-100 is used as the ID dataset. (a) Effect of different adjustment strategies for $\alpha$. (b) Impact of temperature initialization. (c) Evolution of the learned temperature $T$ during training, showing a consistent convergence pattern.}
    \label{fig:ablation}
\end{figure*}

\noindent\textbf{Evaluation Metrics:} Following the setting of \cite{OE:conf/iclr/HendrycksMD19}, we employ FPR95, AUROC and ID ACC for evaluation, where FPR95 refers to the false positive rate of OOD samples when the true positive rate of ID samples is fixed at 95\%, and AUROC is a widely used metric that quantifies the overall separability between ID and OOD samples, based on the receiver operating characteristic curve. 

\noindent\textbf{Baselines:} We select three categories of OOD detection methods for experiments: (1) Post-hoc methods, including MSP~\cite{MSP:conf/iclr/HendrycksG17} and  Energy~\cite{Energy:conf/nips/LiuWOL20}; (2) Training-time regularization methods, including LogitNorm~\cite{LogitNorm:conf/icml/WeiXCF0L22}, T2FNorm~\cite{T2FNorm:conf/cvpr/RegmiPDGSB22}, UM/UMAP~\cite{UM:conf/icml/ZhuLYLX023}, and PSKD~\cite{pskd:conf/icml/0074X25}; (3) Outlier exposure methods, including OE~\cite{OE:conf/iclr/HendrycksMD19}, MixOE~\cite{MixOE:conf/wacv/ZhangILCL23}, DAL~\cite{DAL:conf/nips/WangFZLLH23}, DOE~\cite{DOE:conf/iclr/WangY0DKLH023} and OCL \cite{ocl:conf/aaai/MiaoP0LZ24}. Furthermore, we explore the impact of different score functions, including ODIN~\cite{ODIN:conf/iclr/LiangLS18}, ASH~\cite{ASH:conf/iclr/DjurisicBAL23}, REACT~\cite{react:conf/nips/SunGL21}, DICE~\cite{dice:conf/eccv/SunL22a}, SCALE~\cite{scale:conf/iclr/XuCFY24}, SHE~\cite{she:conf/iclr/ZhangF0DLWLH023}, GEN~\cite{gen:conf/cvpr/LiuLZ23} and RankFeat~\cite{song2024rankfeatweight}.

\begin{table*}[t]
\centering
\caption{OOD detection performance with a series of fixed temperature values on CIFAR-10 and CIFAR-100 using MSP. The best results are highlighted in \textbf{bold}, and the second-best are \underline{underlined}. $\uparrow$ indicates higher is better, while $\downarrow$ indicates lower is better.}
\label{tab:pseudo_label}
\resizebox{\textwidth}{!}{
\begin{tabular}{lcccccccccc}
\toprule
 & \multicolumn{5}{c}{\textbf{CIFAR-10}} & \multicolumn{5}{c}{\textbf{CIFAR-100}} \\
\cmidrule(lr){2-6} \cmidrule(lr){7-11}
Method
& \multicolumn{2}{c}{Near-OOD}
& \multicolumn{2}{c}{Far-OOD}
& ID ACC$\uparrow$
& \multicolumn{2}{c}{Near-OOD}
& \multicolumn{2}{c}{Far-OOD}
& ID ACC$\uparrow$ \\
 & FPR95$\downarrow$ & AUROC$\uparrow$
 & FPR95$\downarrow$ & AUROC$\uparrow$
 & 
 & FPR95$\downarrow$ & AUROC$\uparrow$
 & FPR95$\downarrow$ & AUROC$\uparrow$
 &  \\
\midrule
Random-Hard          & 20.54\tiny$\pm$0.13 & 94.76\tiny$\pm$0.20 & 15.47\tiny$\pm$0.91 & 95.44\tiny$\pm$0.69 & {94.53\tiny$\pm$0.19} & 30.99\tiny$\pm$0.36 & \bf{88.49\tiny$\pm$0.13} & 51.23\tiny$\pm$2.16 & 83.60\tiny$\pm$1.29 & {76.79\tiny$\pm$0.40} \\
Random-Soft            & 19.20\tiny$\pm$1.09 & 95.01\tiny$\pm$0.35 & 14.04\tiny$\pm$0.61 & 95.38\tiny$\pm$0.27 & 94.48\tiny$\pm$0.25 & \underline{30.39\tiny$\pm$0.19} & \underline{88.36\tiny$\pm$0.11} & 50.19\tiny$\pm$4.28 & 84.11\tiny$\pm$2.23 & \bf77.10\tiny$\pm$0.31 \\
Fix $T=3.5$            & 19.48\tiny$\pm$0.84 & 94.94\tiny$\pm$0.30 & 14.01\tiny$\pm$1.08 & 95.98\tiny$\pm$0.33 & 94.53\tiny$\pm$0.08 & 30.71\tiny$\pm$0.51 & 88.26\tiny$\pm$0.19 & 51.57\tiny$\pm$1.79 & 83.60\tiny$\pm$0.96 & \underline{76.81\tiny$\pm$0.38} \\
Fix $T=4.5$          & {18.70\tiny$\pm$0.93} & {95.07\tiny$\pm$0.31} & \underline{11.35\tiny$\pm$1.24} & \underline{96.61\tiny$\pm$0.42} & 94.50\tiny$\pm$0.26 & 31.92\tiny$\pm$1.82 & 87.49\tiny$\pm$0.99 & 46.30\tiny$\pm$9.58 & {85.48\tiny$\pm$3.89} & 76.35\tiny$\pm$0.98 \\
Fix $T=5.5$             & 19.37\tiny$\pm$1.08 & 94.93\tiny$\pm$0.35 & 13.89\tiny$\pm$0.65 & 95.63\tiny$\pm$0.51 & 94.66\tiny$\pm$0.11 & 32.42\tiny$\pm$2.31 & 87.41\tiny$\pm$1.01 & \bf{45.64\tiny$\pm$12.77} & 84.99\tiny$\pm$5.14 & 76.39\tiny$\pm$0.97 \\
\textit{AOE-At}      & \underline{18.37\tiny$\pm$0.57} & \underline{95.21\tiny$\pm$0.09} & {11.88\tiny$\pm$0.15} & 96.25\tiny$\pm$0.42 & \bf94.82\tiny$\pm$0.19 & 30.56\tiny$\pm$0.14 & 88.24\tiny$\pm$0.07 & \underline{45.71\tiny$\pm$2.88} & \underline{85.97\tiny$\pm$1.06} & 76.21\tiny$\pm$0.01 \\
\textit{AOE-Jt}      & \bf17.44\tiny$\pm$0.05 & \bf95.64\tiny$\pm$0.03 & \bf10.62\tiny$\pm$1.50 & \bf97.04\tiny$\pm$0.47 & \underline{94.79\tiny$\pm$0.15} & \bf30.09\tiny$\pm$0.21 & {88.32\tiny$\pm$0.10} & 46.12\tiny$\pm$2.26 & \bf86.04\tiny$\pm$0.38 & 76.90\tiny$\pm$0.14 \\
\bottomrule
\end{tabular}}
\end{table*}

\begin{table*}[t]
\centering
\caption{OOD detection performance of OE and its AOE-enhanced variant across different backbone architectures on CIFAR-100. AOE consistently improves performance across architectures, with more significant gains observed in far-OOD settings}
\label{tab:cifar100_backbone}
\begin{tabular}{lcccccccc}
\toprule
\multirow{3}{*}{\textbf{Backbone}} 
& \multicolumn{4}{c}{\textbf{Near-OOD}} 
& \multicolumn{4}{c}{\textbf{Far-OOD}} \\
\cmidrule(lr){2-5} \cmidrule(lr){6-9}
& \multicolumn{2}{c}{FPR95 $\downarrow$} & \multicolumn{2}{c}{AUROC $\uparrow$} 
& \multicolumn{2}{c}{FPR95 $\downarrow$} & \multicolumn{2}{c}{AUROC $\uparrow$} \\
& OE & AOE & OE & AOE & OE & AOE & OE & AOE \\
\midrule
WideResNet  &30.27\tiny$\pm$0.81 &30.02\tiny$\pm$0.36  &89.38\tiny$\pm$0.15 &89.34\tiny$\pm$0.06 &36.30\tiny$\pm$ 0.58 &35.04\tiny$\pm$0.82  &88.92\tiny$\pm$ 0.36 &89.68\tiny$\pm$0.32 \\
ResNet-18   &30.73\tiny$\pm$0.11  &30.09\tiny$\pm$0.21 &88.30\tiny$\pm$0.10 & 88.32\tiny$\pm$0.10 &54.82\tiny$\pm$2.79 &46.12\tiny$\pm$2.26 &81.41\tiny$\pm$1.49 &86.04\tiny$\pm$0.38 \\
DenseNet      &32.04\tiny$\pm$0.42 &31.62\tiny$\pm$0.52 & 87.05\tiny$\pm$0.34 &86.81\tiny$\pm$0.02 &52.34\tiny$\pm$2.37  &50.21\tiny$\pm$0.32 &81.84\tiny$\pm$0.57   &82.73\tiny$\pm$1.04 \\
\bottomrule
\end{tabular}
\end{table*}

\begin{figure}[!t]
    \centering
    \begin{minipage}[b]{0.475\linewidth}
        \centering
        \includegraphics[width=\linewidth]{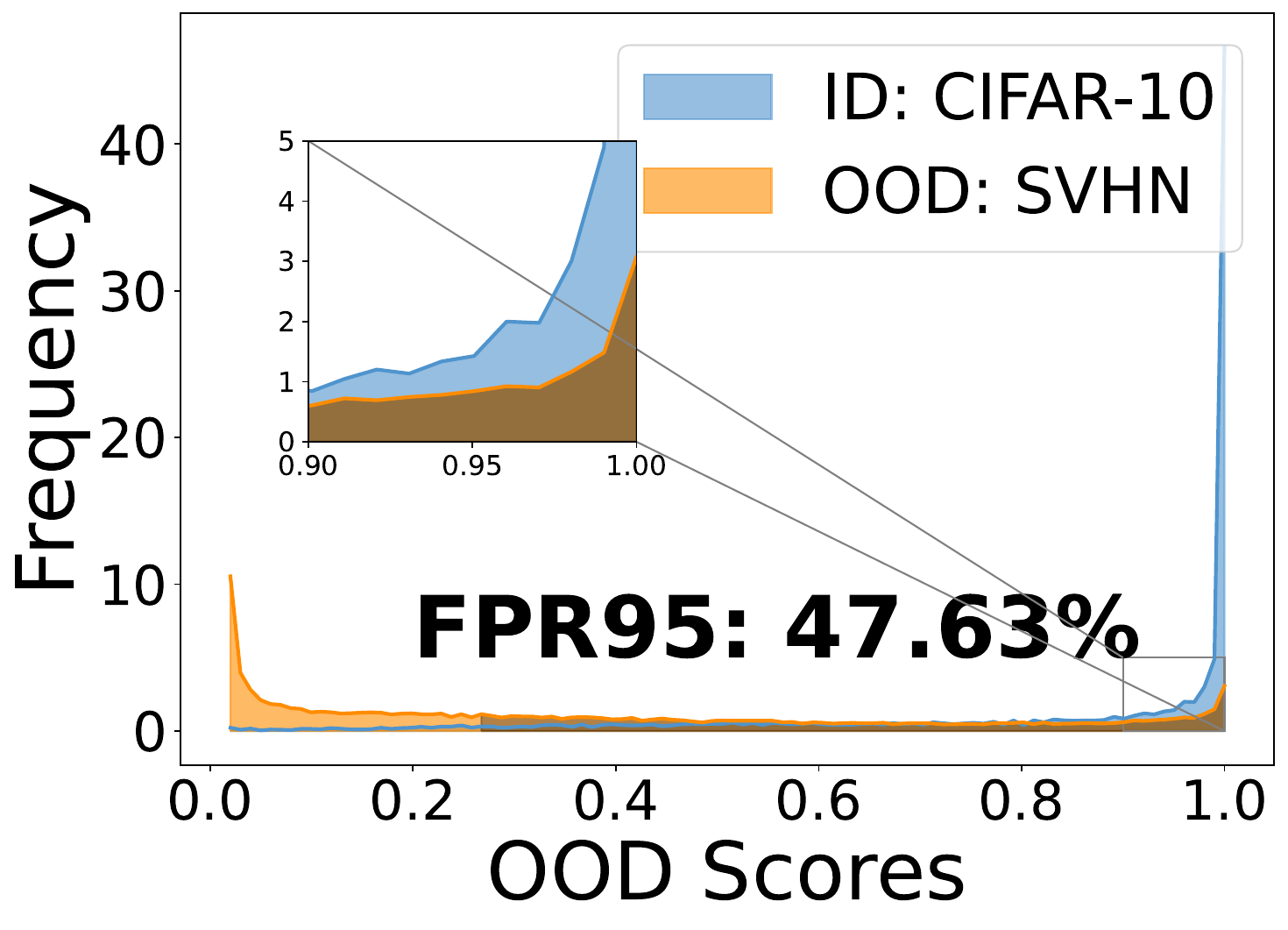}
        (a) OE
    \end{minipage}
    \begin{minipage}[b]{0.475\linewidth}
        \centering
        \includegraphics[width=\linewidth]{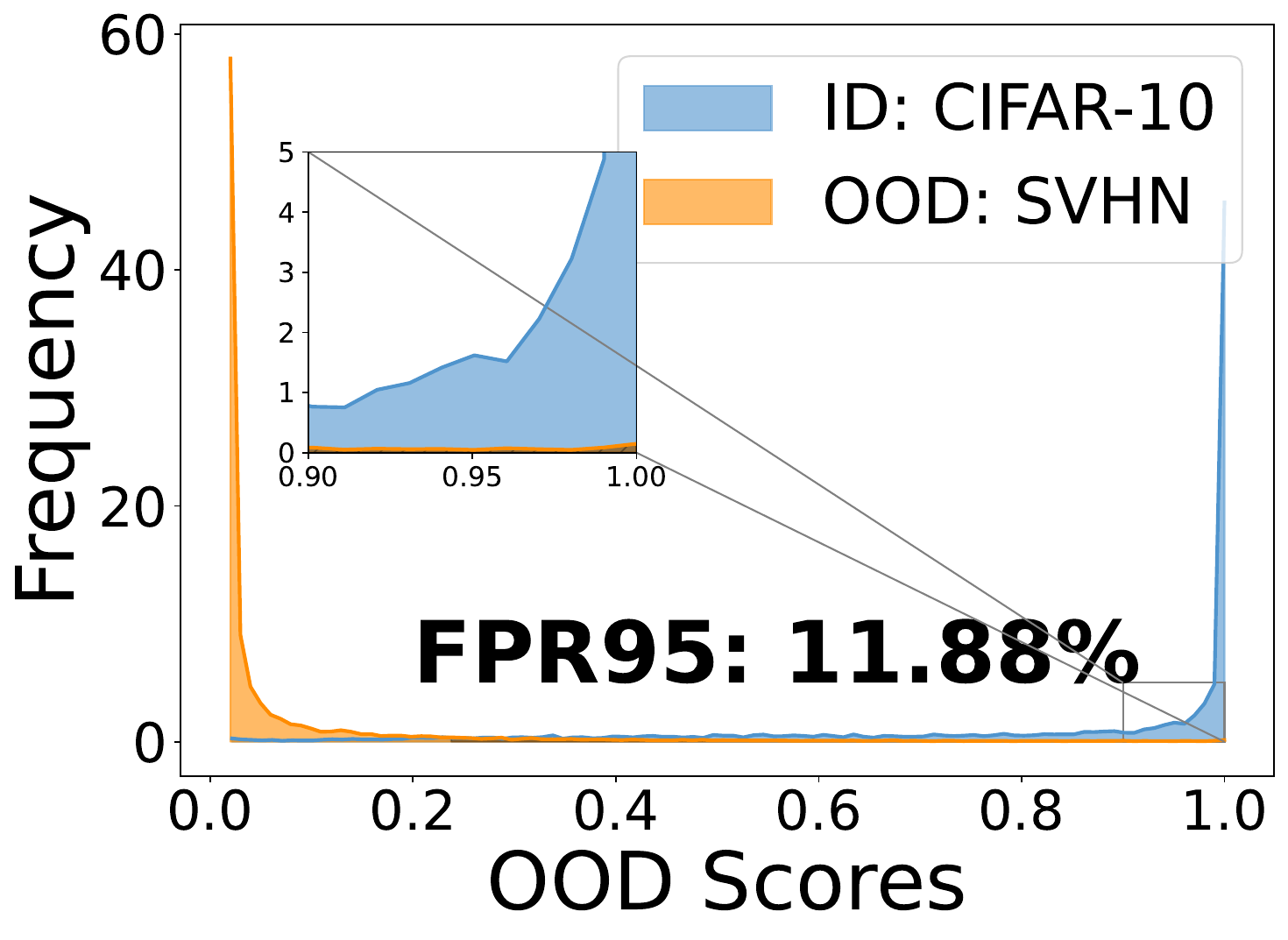}
        (b) AOE
    \end{minipage}\\
    \begin{minipage}[b]{0.475\linewidth}
        \centering
        \includegraphics[width=\linewidth]{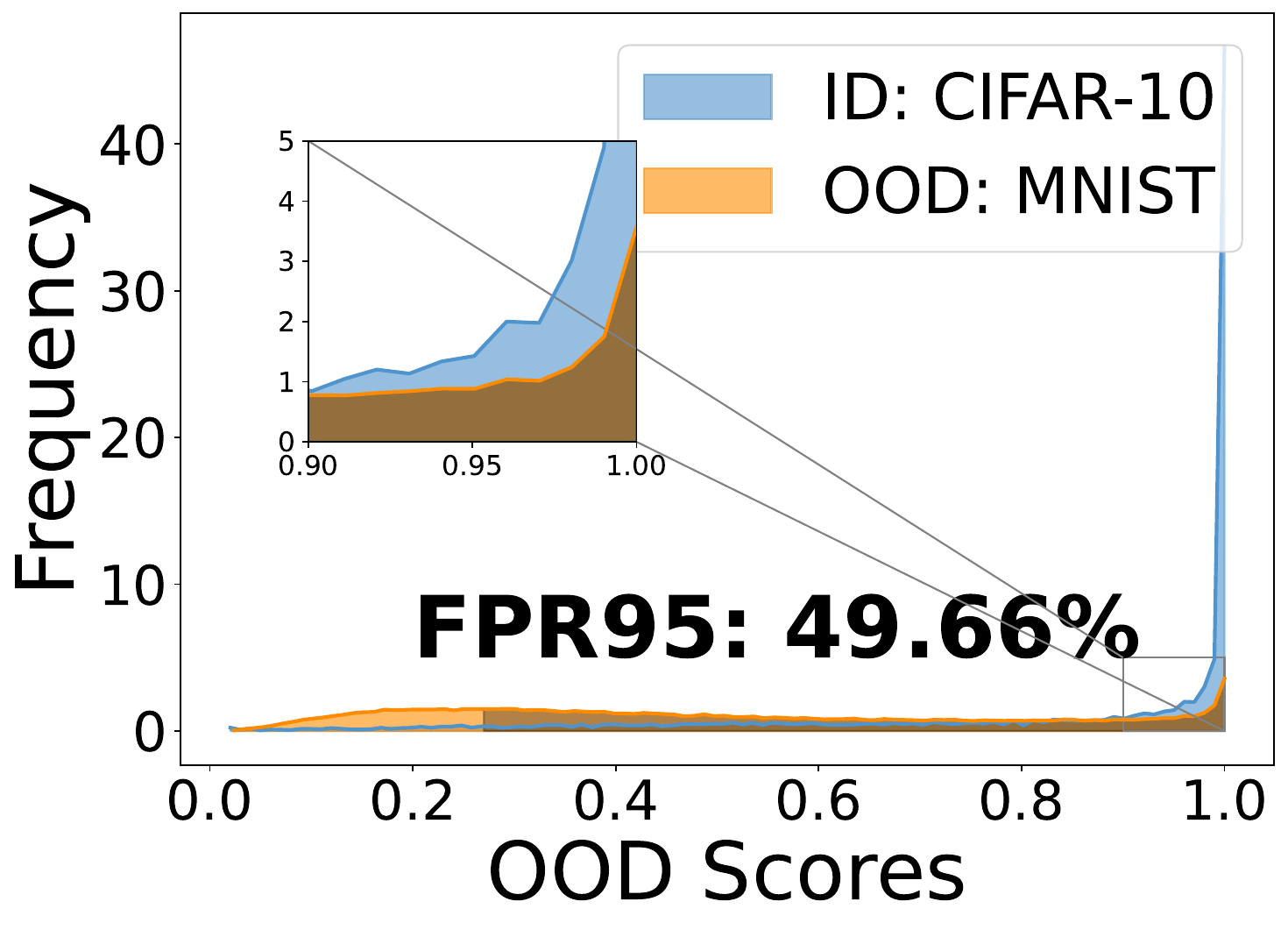}
        (c) OE
    \end{minipage}
    \begin{minipage}[b]{0.475\linewidth}
        \centering
        \includegraphics[width=\linewidth]{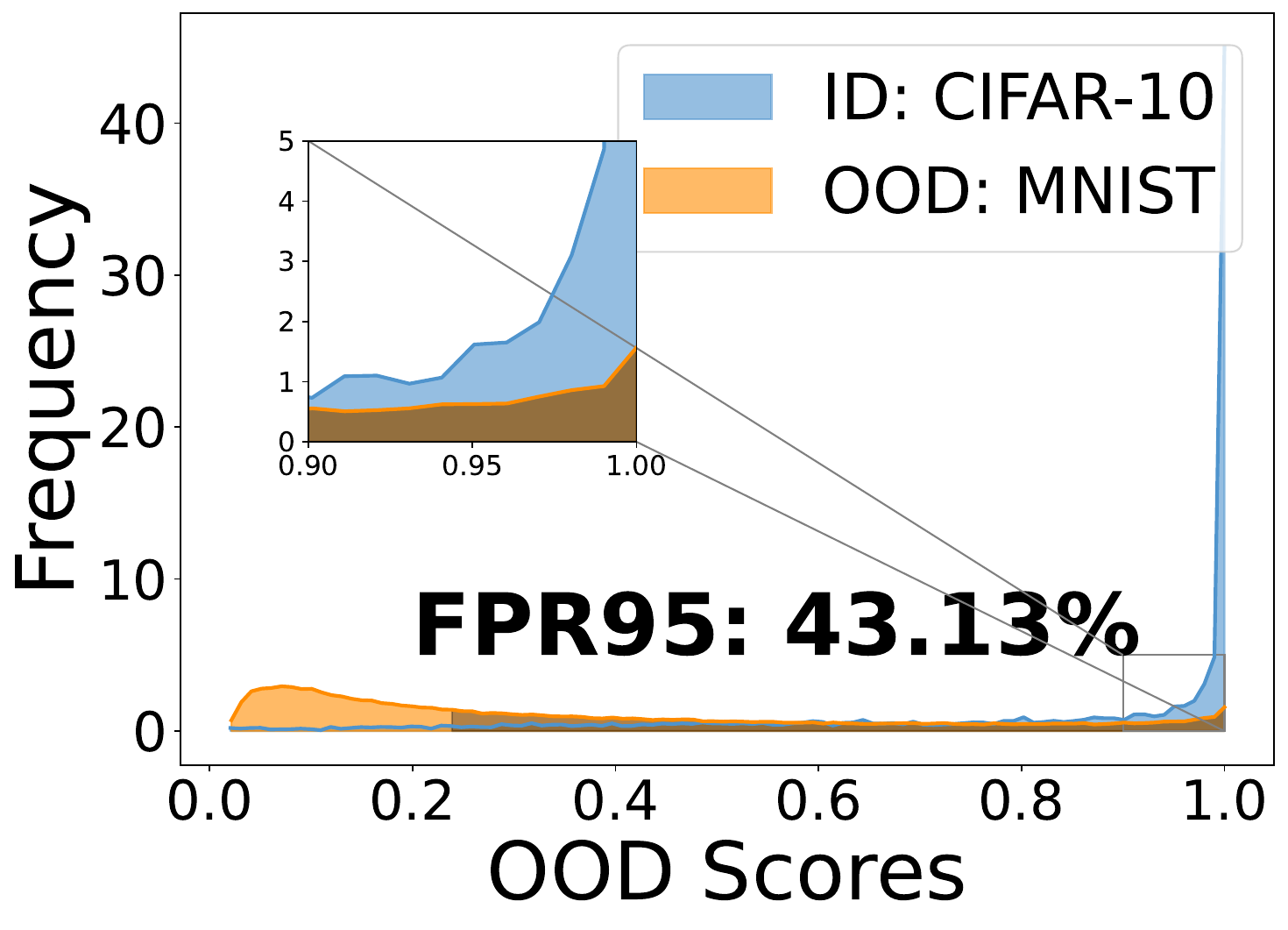}
        (d)  AOE
    \end{minipage}
    \caption{\revised{OOD score distributions of OE and AOE. CIFAR-10 is used as the ID dataset, while SVHN and MNIST are used as OOD datasets. The histograms compare the confidence score distributions of OOD samples under OE and AOE, with emphasis on the high-score region. Compared with OE, AOE produces fewer high-confidence OOD predictions and achieves lower FPR95, indicating improved separation between ID and OOD samples.}}
    \label{fig:ood_distribution}
\end{figure}

\subsection{Main Results}

We report the OOD detection performance of AOE and SOTA methods on large-scale datasets in \cref{tab:imagenet_ood} and small-scale datasets in \cref{tab:cifar_ood}. For all methods, the mean performance and standard deviation are computed over multiple independent runs with different random initializations. Several observations can be drawn from the experimental results. (1) OE-based methods, including our proposed AOE, generally outperform training-time regularization methods and score-based methods. This advantage can be primarily attributed to the incorporation of outlier samples during training, which enhances the model’s ability to discriminate OOD samples from ID samples. (2) However, conventional OE-based methods often compromise ID classification performance and fail to fully exploit the relations between OOD samples and ID categories, as theoretically analyzed in \cref{thm:delta_margin_upper}, particularly when assigning an additional class to OOD samples. In contrast, as stated in \cref{thm:T_mitigate}, our method effectively mitigates this over-softening issue, leading to improved overall performance. As a result, compared with existing methods, AOE achieves superior performance across nearly all evaluated scenarios. (3) \textit{AOE-Jt} of the temperature parameter \( T \) outperforms \textit{AOE-At}, as it allows the model to co-adapt the temperature alongside its feature representations, achieving more accurate alignment with ID decision boundaries. Additional plug-in results on existing OE methods are shown in \cref{tab:plug-play_oe}, confirming the generality and compatibility of AOE.

\begin{table*}[t]
\centering
\caption{OOD detection performance across different scoring functions on CIFAR-100. AOE consistently improves OOD detection across a wide range of scoring methods, achieving lower FPR95 and higher AUROC in most cases, with particularly significant gains in far-OOD scenarios.}
\label{tab:cifar100_ood_scores}
\resizebox{\textwidth}{!}{
\begin{tabular}{lcccccccccccc}
\toprule
\multirow{3}{*}{\textbf{OOD Scores}} 
& \multicolumn{6}{c}{\textbf{Near-OOD}} 
& \multicolumn{6}{c}{\textbf{Far-OOD}} \\
\cmidrule(lr){2-7} \cmidrule(lr){8-13}
& \multicolumn{3}{c}{FPR95 $\downarrow$} & \multicolumn{3}{c}{AUROC $\uparrow$} 
& \multicolumn{3}{c}{FPR95 $\downarrow$} & \multicolumn{3}{c}{AUROC $\uparrow$} \\
& OE & \textit{AOE-At} & \textit{AOE-Jt} & OE & \textit{AOE-At} & \textit{AOE-Jt} & OE & \textit{AOE-At} & \textit{AOE-Jt} & OE & \textit{AOE-At} & \textit{AOE-Jt}\\
\midrule
Energy & 30.34\tiny$\pm$0.21 & 30.49\tiny$\pm$0.33 & 30.28\tiny$\pm$0.71 & 88.42\tiny$\pm$0.05 & 88.25\tiny$\pm$0.21 & 88.18\tiny$\pm$0.37 & 50.59\tiny$\pm$0.76 & 39.26\tiny$\pm$2.26 & 41.10\tiny$\pm$1.38 & 84.10\tiny$\pm$0.85 & 88.34\tiny$\pm$0.86 & 87.86\tiny$\pm$0.28 \\
ODIN   & 40.22\tiny$\pm$0.66 & 39.09\tiny$\pm$2.58 & 38.46\tiny$\pm$0.23 & 86.19\tiny$\pm$0.51 & 86.46\tiny$\pm$0.78 & 86.82\tiny$\pm$0.13 & 49.08\tiny$\pm$3.40 & 45.79\tiny$\pm$2.25 & 48.96\tiny$\pm$1.75 & 84.78\tiny$\pm$1.32 & 86.70\tiny$\pm$0.38 & 85.93\tiny$\pm$0.45 \\
ASH    & 36.46\tiny$\pm$1.41 & 35.83\tiny$\pm$0.79 & 36.88\tiny$\pm$0.84& 85.74\tiny$\pm$1.23 & 85.58\tiny$\pm$0.62 & 85.08\tiny$\pm$0.66 & 55.53\tiny$\pm$0.84 & 42.49\tiny$\pm$2.14 & 44.72\tiny$\pm$1.76 & 82.78\tiny$\pm$1.01 & 87.10\tiny$\pm$0.92 & 86.29\tiny$\pm$0.62  \\
REACT  & 31.09\tiny$\pm$1.39 & 33.37\tiny$\pm$2.10 & 31.46\tiny$\pm$1.64 & 87.92\tiny$\pm$0.49 & 86.80\tiny$\pm$1.33 & 87.60\tiny$\pm$0.62 & 49.62\tiny$\pm$0.98 & 40.33\tiny$\pm$0.59 & 40.72\tiny$\pm$2.51 & 84.36\tiny$\pm$0.85 & 88.19\tiny$\pm$0.49 & 88.03\tiny$\pm$0.37  \\
DICE   & 31.44\tiny$\pm$0.12 & 31.32\tiny$\pm$0.66 & 31.08\tiny$\pm$1.30 & 87.32\tiny$\pm$0.16 & 87.34\tiny$\pm$0.32 & 87.21\tiny$\pm$0.50 & 49.45\tiny$\pm$0.63 & 37.10\tiny$\pm$1.93 & 40.14\tiny$\pm$0.83 & 84.25\tiny$\pm$0.64 & 88.38\tiny$\pm$0.88 & 87.63\tiny$\pm$0.15  \\
SCALE  & 36.98\tiny$\pm$0.29 & 30.74\tiny$\pm$0.94 & 32.86\tiny$\pm$3.30 & 85.71\tiny$\pm$0.10 & 88.11\tiny$\pm$0.32 & 87.13\tiny$\pm$1.17 & 55.49\tiny$\pm$1.52 & 38.27\tiny$\pm$1.89 & 41.16\tiny$\pm$3.64 & 83.43\tiny$\pm$0.07 & 88.69\tiny$\pm$0.84 & 87.85\tiny$\pm$0.94 \\
SHE    & 32.47\tiny$\pm$0.26 & 31.94\tiny$\pm$1.05 & 32.31\tiny$\pm$1.43 & 86.67\tiny$\pm$0.15 & 86.81\tiny$\pm$0.46 & 86.35\tiny$\pm$0.65 & 55.92\tiny$\pm$1.44 & 41.86\tiny$\pm$1.64 & 44.47\tiny$\pm$1.30 & 82.24\tiny$\pm$1.02 & 87.13\tiny$\pm$0.82 & 86.44\tiny$\pm$0.13 \\
GEN    & 30.18\tiny$\pm$0.31 & 30.48\tiny$\pm$0.32 & 30.25\tiny$\pm$0.72 & 88.54\tiny$\pm$0.21 & 88.37\tiny$\pm$0.23 & 88.20\tiny$\pm$0.37 & 51.05\tiny$\pm$1.11 & 40.01\tiny$\pm$2.98 & 41.14\tiny$\pm$1.40 & 83.97\tiny$\pm$0.80 & 88.24\tiny$\pm$0.98 & 87.86\tiny$\pm$0.29 \\
RankFeat    & 50.38\tiny$\pm$10.08 & 48.37\tiny$\pm$1.81 & 50.33\tiny$\pm$0.68 & 77.64\tiny$\pm$4.37 & 76.86\tiny$\pm$1.16 & 77.78\tiny$\pm$1.83 & 90.43\tiny$\pm$3.29 & 65.95\tiny$\pm$12.66 & 70.65\tiny$\pm$13.77& 55.36\tiny$\pm$8.49 & 73.44\tiny$\pm$8.62 & 68.82\tiny$\pm$12.74 \\
\bottomrule
\end{tabular}
}
\end{table*}

\subsection{Ablation Study}

\noindent\textbf{Impact of Adjustment Strategies for $\alpha$.} Balancing coefficient $\alpha$ serves as the most influential factor for our model’s effectiveness. As discussed in \cref{sec:algo}, we investigate strategies for adjusting the value of $\alpha$ during different stages of training. We designed four scheduling strategies for the parameter $\alpha$: (1) a fixed-value strategy, i.e., $\alpha_t=c$, where $t$ and $c$ denote the epoch and constant; (2) an exponential scheduling strategy, i.e., $\alpha_t=1-e^{-t/35}$; (3) a cosine annealing strategy, i.e., $\alpha_t=0.5-\mathrm{cos}((t+1)\pi/100)/2$; and (4) a linear scheduling strategy, i.e., $\alpha_t=\min\{1,(t+1)/100\}$. For the first strategy, we set $c=\{0.5,1\}$ for experiments. The results are shown in \cref{fig:ablation}(a), where ``F-0.5'' and ``F-1.0'' denote the fixed-value strategy with different $c$, ``Exp'', ``Cos'', and ``Linear'' respectively denote the exponential scheduling strategy, cosine annealing strategy and linear scheduling strategy. We can find that different scheduling strategies for $\alpha$ result in comparable performance, with the cosine annealing strategy achieving the best overall results.

\noindent\textbf{Impact of Temperature Initialization.} Temperature $T$ is treated as a learnable parameter. We report the OOD detection performance with different initialization values for $T\in[1.5,10]$ in \cref{fig:ablation}(b). Furthermore, we illustrate the change of $T$ during model training in \cref{fig:ablation}(c). The results reveal that the initial value of \( T \) significantly affects the final performance. Interestingly, regardless of the initialization, the learned temperature consistently follows a similar trajectory, converging to a range between 1.5 and 2, and then gradually decreasing towards the end of training. This indicates that early-stage temperature values play a critical role in guiding effective optimization.

\begin{figure}[!t]
    \centering
    \begin{minipage}[b]{0.98\linewidth}
        \centering
        \includegraphics[width=0.24\linewidth]{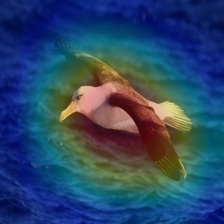}
        \includegraphics[width=0.24\linewidth]{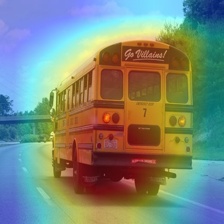}
        \includegraphics[width=0.24\linewidth]{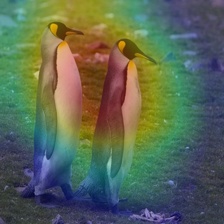}
        \includegraphics[width=0.24\linewidth]{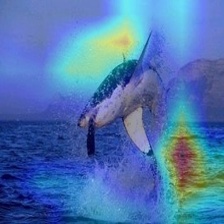}
    \end{minipage}
    (a) OE
    \begin{minipage}[b]{0.98\linewidth}
        \centering
        \includegraphics[width=0.24\linewidth]{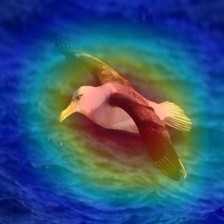}
        \includegraphics[width=0.24\linewidth]{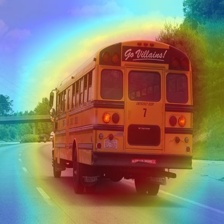}
        \includegraphics[width=0.24\linewidth]{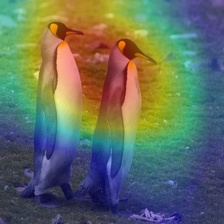}
        \includegraphics[width=0.24\linewidth]{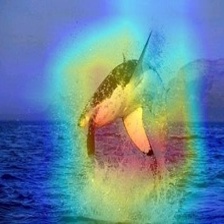}
    \end{minipage}
    (b) AOE
    \caption{\revised{Grad-CAM visualizations for OOD samples under OE and AOE. Compared with OE, AOE produces more focused and semantically meaningful activation regions, indicating improved uncertainty modeling for OOD samples.}}
    \label{fig:grad_cam_visual}
\end{figure}

\subsection{Further Analysis}
\noindent\textbf{Impact of Learnable Temperature.} We evaluate the impact of a learnable temperature by comparing AOE with alternative pseudo-labeling strategies, including random hard labels, random soft labels, and fixed temperature coefficients. The results indicate that AOE's performance gains do not stem from arbitrary label smoothing, but from preserving relations during calibration. As shown in \cref{tab:pseudo_label}, random hard and soft labels yield only marginal improvements over standard OE, while fixed-temperature calibration offers moderate gains but suffers from limited robustness. In contrast, AOE, with adaptive pseudo-labels and a learnable temperature, consistently outperforms all baselines, underscoring the effectiveness of relations-preserving calibration for reliable OOD detection.

\noindent\textbf{Impact of OOD Score Function.} Since our study is orthogonal to the choice of score functions, we investigate the impact of different scoring functions on the final OOD detection performance in this section. We report the results on the CIFAR-100 benchmark with various score functions in \cref{tab:cifar100_ood_scores}. We observe that standard OE-based methods exhibit degraded performance when combined with different scoring functions. This degradation arises because these methods distort the model’s confidence landscape, resulting in misaligned score distributions at test time. In contrast, AOE achieves consistently strong performance across diverse scoring functions by preserving relations during training, thereby ensuring better compatibility with OOD score.

\noindent\revised{\textbf{Impact of Auxiliary Outlier Sources.}
To examine the influence of auxiliary outlier sources, we evaluate two representative sources widely used in OE-based studies, including 300K random images sampled from 80 Million Tiny Images~\cite{macs} and ImageNet-RC~\cite{POME}. As shown in \cref{fig:futher}(a), AOE maintains a consistent performance advantage over OE under different auxiliary outlier sources. This observation suggests that the improvement is not merely caused by a particular outlier exposure dataset, but is more closely related to the proposed recalibration of outlier labels, which adaptively preserves the predictive relations between auxiliary outliers and ID categories.
}

\noindent\revised{\textbf{Impact of the Number of Auxiliary Outliers.}
We further investigate the robustness of AOE to the number of exposed auxiliary outliers. Specifically, we randomly discard auxiliary OOD samples according to different mask ratios and compare AOE with OE under the same training protocol. As shown in \cref{fig:futher}(b), AOE consistently outperforms OE across different numbers of auxiliary outliers. Moreover, the results indicate that simply increasing the number of exposed outliers does not necessarily lead to the best performance, suggesting that the informativeness of auxiliary outliers is also crucial.
}

\begin{figure}[t]
    \centering
    \begin{minipage}[b]{\linewidth}
        \centering
        \includegraphics[width=0.32\linewidth,height=0.4\linewidth]{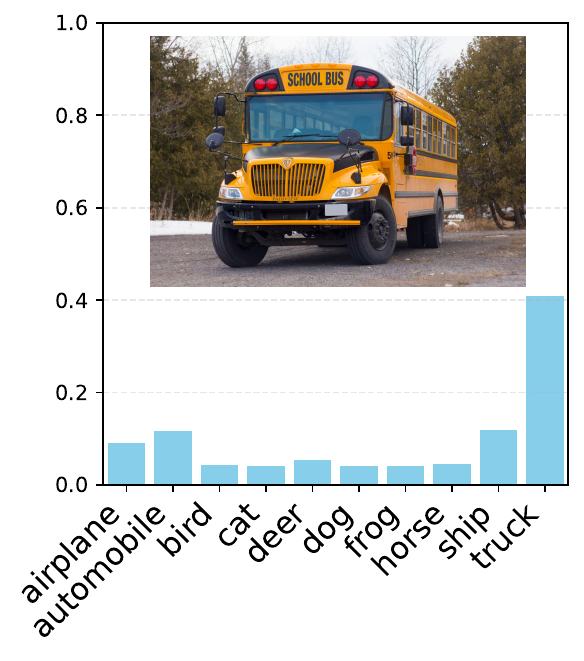}
        \includegraphics[width=0.32\linewidth,height=0.4\linewidth]{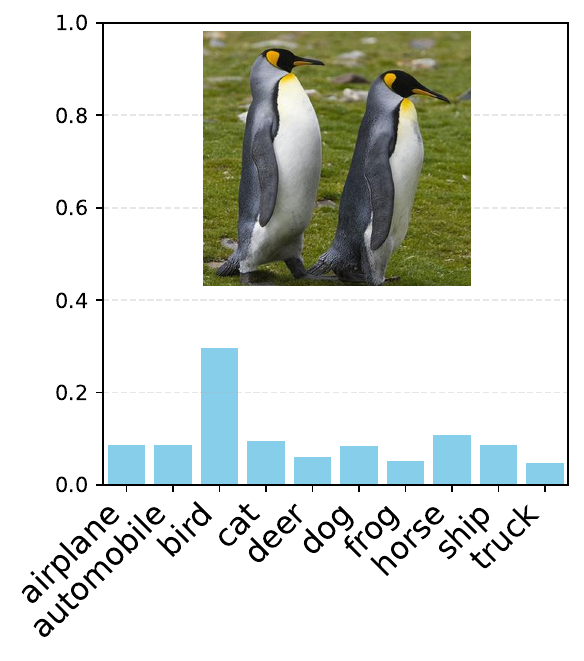}
        \includegraphics[width=0.32\linewidth,height=0.4\linewidth]{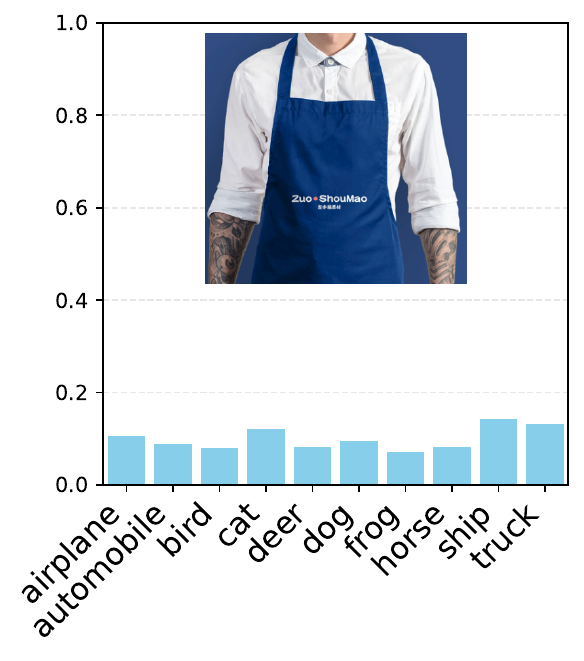}
    \end{minipage} \\
    \caption{
    \revised{Case studies of temperature-scaled outlier targets generated by AOE. Each bar plot shows the soft target distribution over CIFAR-10 ID categories for a representative OOD sample. AOE assigns relatively higher probabilities to semantically related ID categories while maintaining uncertainty for ambiguous OOD samples.}
    }
    \label{fig:pesual_label}
\end{figure}

\noindent\textbf{Impact of Model Architecture.} As shown in \cref{tab:cifar100_backbone}, AOE consistently improves over OE across all backbones, with more pronounced gains on far-OOD. The improvements are especially significant for ResNet-18, while remaining stable for stronger architectures such as WideResNet and DenseNet. These results indicate that AOE is robust to architectural choices and particularly beneficial for weaker models and harder OOD settings.

\noindent\textbf{Visualization.} We further present the histograms of OOD scores for our proposed method and the OE baseline for comparison. \cref{fig:ood_distribution}(a) and \cref{fig:ood_distribution}(b) use CIFAR-10 as the ID dataset and SVHN as the OOD dataset, whereas \cref{fig:ood_distribution}(c) and \cref{fig:ood_distribution}(d) use MNIST as the OOD dataset. OE assigns uniform targets to OOD samples, enforcing equal class uncertainty and resulting in suppressed the model’s ability to distinguish ambiguous samples and results in overconfident responses. \cref{fig:ood_distribution} shows substantial overlap between ID and OOD scores, indicating poor separability. In contrast, AOE replaces uniform supervision with adaptive pseudo-labels modulated by a learnable temperature, which introduces sample-dependent uncertainty into the training signal. This mechanism encourages calibrated confidence for hard OOD samples, leading to a noticeable reduction in high-confidence OOD predictions and a clearer separation between ID and OOD score distributions.

\noindent\textbf{Case Study.} We examine the pseudo-label distributions of several real OOD samples with respect to ID categories to gain deeper insight into the behavior of our method. The results are presented in \cref{fig:pesual_label}. For the OOD sample labeled as “School Bus,” our method assigns relatively high confidence to the semantically related ID class “Truck.” A similar pattern is observed for the sample labeled as “King Penguin.” In contrast, for samples such as “Apron”, the model produces a more uniform distribution over ID categories, reflecting higher uncertainty. We further visualize the corresponding Grad-CAM \cite{gradcam:journals/corr/abs-2402-00909} activation maps in \cref{fig:grad_cam_visual} to examine how different supervision strategies affect the model’s attention. Under OE, the attention is often diffuse, reflecting the absence of meaningful guidance from uniform target supervision. In contrast, AOE produces more focused and semantically coherent attention regions, aligning well with the structure of the learned pseudo-labels. 

\noindent\revised{\textbf{Computational Overhead.}
To further evaluate the practical efficiency of AOE, we analyze the additional overhead introduced by the proposed objective. As defined in \cref{eq:ood_objective}, AOE only introduces a learnable temperature $T$ to recalibrate outlier labels from the original OOD logits. This design keeps the model architecture unchanged and reuses the same backbone computation as standard OE. The extra training cost mainly comes from temperature scaling and distribution alignment on existing OOD logits, which is minor compared with the standard forward and backward propagation. Moreover, since $T$ is used only for training-time label recalibration and is discarded during inference, AOE has the same inference procedure and inference cost as the corresponding OE-based method.}

\begin{figure}[!t]
\centering
\begin{minipage}{0.49\linewidth}
\centering
\includegraphics[width=\linewidth]{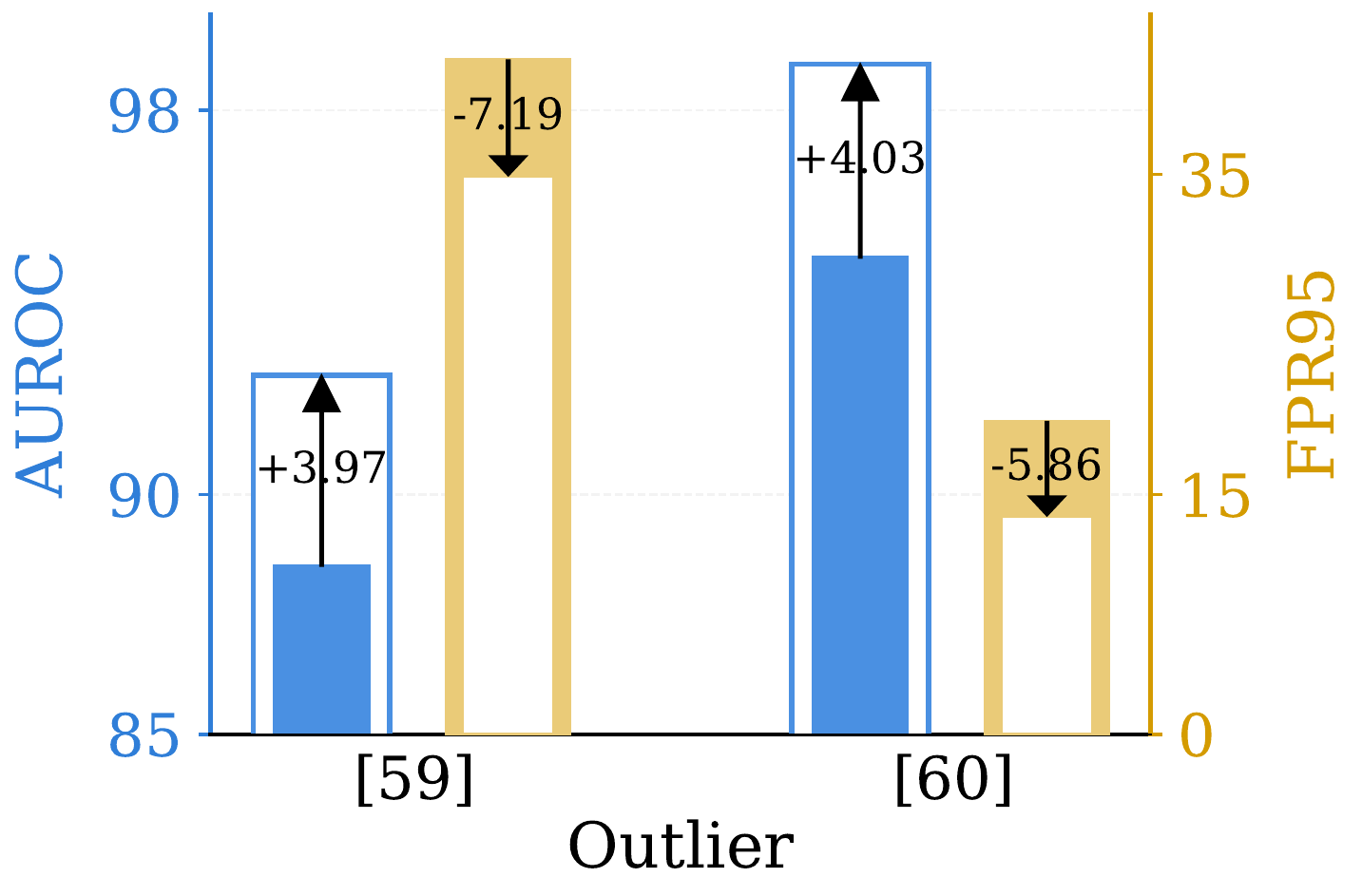}
(a) Auxiliary outlier sources
\end{minipage}
\begin{minipage}{0.49\linewidth}
\centering
\includegraphics[width=\linewidth]{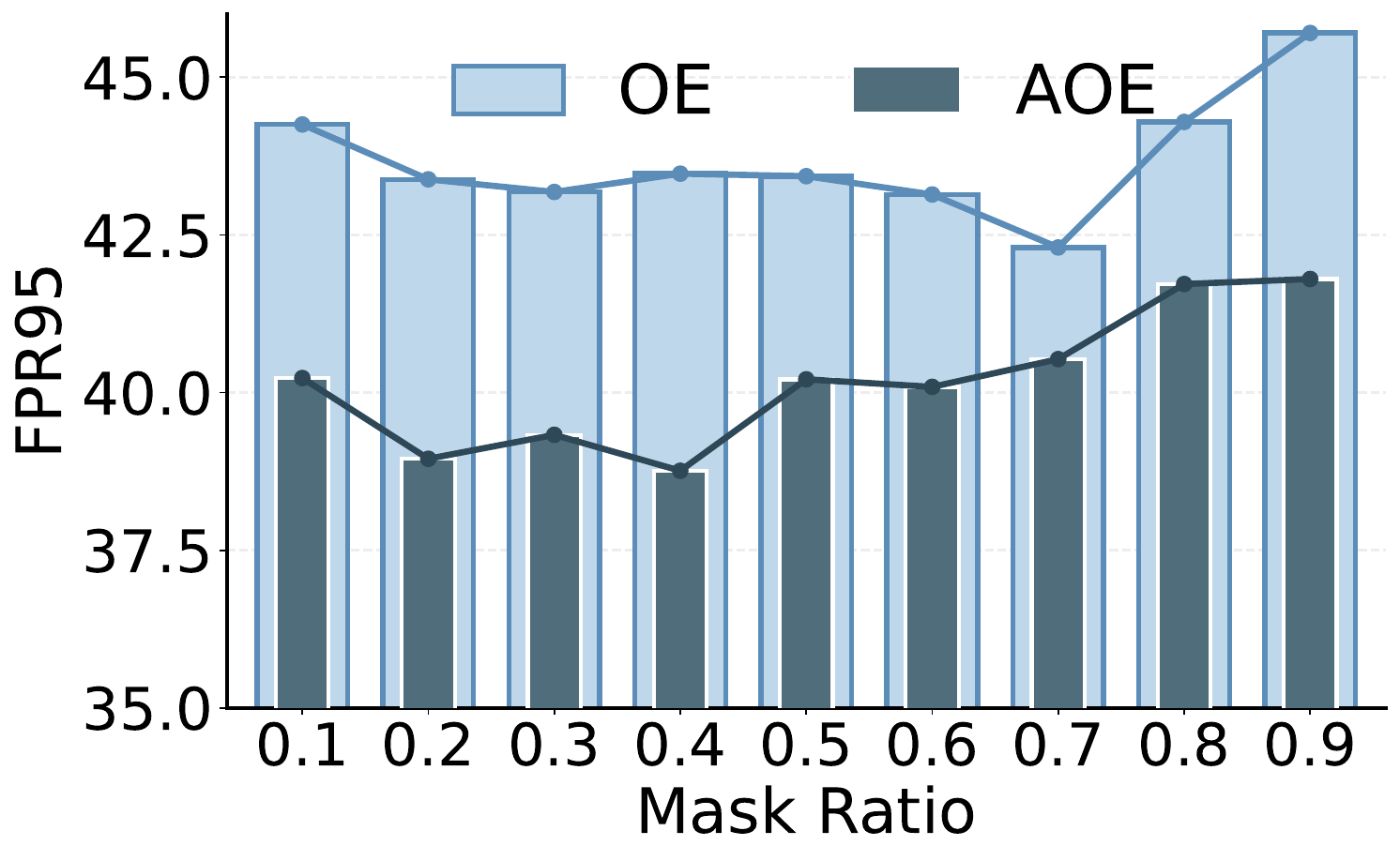}
\vspace{-4pt}
(b) Number of auxiliary outliers
\end{minipage}
\caption{\revised{Further analyses of auxiliary outliers on CIFAR-100. (a) Effect of auxiliary outlier sources. The upward arrows indicate the AUROC increasing from OE to AOE, whereas the downward arrows denote the reduction in FPR95. (b) Effect of the number of auxiliary outliers. AOE consistently outperforms OE in both settings.}}
\label{fig:futher}
\end{figure}

\section{Conclusion}

We identify a fundamental limitation of conventional Outlier Exposure, where uniform supervision on OOD samples induces excessive margin contraction and restricts distributional separation. To address this issue, we propose AOE, which replaces uniform targets with adaptive pseudo-labels derived from temperature-scaled model predictions. Our theoretical analysis establishes a direct connection between the relations of OOD samples with respect to ID categories and OOD separability, and shows that adaptive temperature scaling effectively mitigates the over-softening effect. We further introduce joint and alternating optimization schemes to learn the temperature during training. Extensive experiments demonstrate that AOE consistently improves OOD detection performance while preserving competitive ID Acc, and can be seamlessly integrated with existing OE strategies.

\noindent\revised{\textbf{Limitations and Future Work.}
Despite these improvements, the effectiveness of AOE is closely related to the presence of meaningful predictive relations between auxiliary outliers and ID categories. When auxiliary outliers are substantially distant from the ID label space, the resulting temperature-scaled pseudo-labels may approach uniform targets, thereby reducing the distinguishable advantage over conventional OE. This indicates that AOE is particularly effective when auxiliary outliers contain partial semantic proximity with ID classes. Future work may further strengthen AOE through relation-aware outlier weighting, enabling informative outliers to be more effectively exploited while weakly related samples are adaptively suppressed.
}

\section*{Acknowledgements}
This work was supported by the  NSFC(62276131, 62506168), the Natural Science Foundation of Jiangsu Province of China (BK20240081, BK20251431), the Fundamental Research Funds for the Central Universities (No. 30925010205), the Special Research Project on Teaching Reform of General Artificial Intelligence Courses in Jiangsu Undergraduate Universities (No. ZNT-10), and the Open Project of the Key Laboratory of Modern Agricultural Equipment, Ministry of Agriculture and Rural Affairs, P.R. China.

\section*{Competing interests}
The authors declare that they have no competing interests or financial conflicts to disclose.

\section*{Appendixes}

\subsection*{Theoretical proof}
\label{appendix:proof}

\subsubsection*{Proof of \cref{prop:oversoften}}

We analyze the parameter update under the Outlier Exposure objective using a uniform target $\mathcal{U} = (1/K, \dots, 1/K)$. Let $z = f(x;\theta) \in \mathbb{R}^K$ be the logits for a given sample $x$, and $p = s(z)$ be the predicted softmax probabilities, where $p_k = \frac{e^{z_k}}{\sum_{j=1}^K e^{z_j}}$. The KL divergence loss with respect to the uniform distribution is defined as:
\begin{equation}
    \mathcal{L}_{\mathrm{KL}} = -\frac{1}{K} \sum_{k=1}^K \log p_k + \text{const}.
\end{equation}
The gradient of the loss with respect to the $k$-th logit is:
\begin{equation}
    \frac{\partial \mathcal{L}_{\mathrm{KL}}}{\partial z_k} = p_k - \frac{1}{K}.
\end{equation}
Under a gradient descent update with learning rate $\eta$, the updated logits are $z^+ = z - \eta \nabla_z \mathcal{L}_{\mathrm{KL}}$. For any two categories $i$ and $j$, the difference between their updated logits is:
\begin{align}
    z_i^+ - z_j^+ &= (z_i - z_j) - \eta \left( \frac{\partial \mathcal{L}_{\mathrm{KL}}}{\partial z_i} - \frac{\partial \mathcal{L}_{\mathrm{KL}}}{\partial z_j} \right) \\
    &= (z_i - z_j) - \eta(p_i - p_j).
\end{align}
Assume $z_i > z_j$ without loss of generality. Due to the strict monotonicity of the softmax function, $z_i > z_j \implies p_i > p_j$. Thus, $p_i - p_j > 0$. For a sufficiently small $\eta$ such that $z_i^+ - z_j^+$ maintains the same sign as $z_i - z_j$, we have:
\begin{equation}
    0 < z_i^+ - z_j^+ < z_i - z_j.
\end{equation}
Taking the absolute value, we obtain:
\begin{equation}
    |z_i^+ - z_j^+| = |(z_i - z_j) - \eta(p_i - p_j)| \le |z_i - z_j|.
\end{equation}
This demonstrates that optimizing toward a uniform distribution inherently contracts the relative differences between logits, resulting in the over-softening effect.

\subsubsection*{Proof of \cref{thm:delta_margin_upper}}

Let $z=f(x;\theta)\in\mathbb{R}^K$ denote the logits of an input $x$, and define
\(k^\star(x)=\arg\max_{k\in\{1,\dots,K\}} z_k\) and
\(m(x;\theta)\triangleq z_{k^\star(x)}-\max_{j\neq k^\star(x)} z_j\).
We adopt the widely used maximum softmax probability (MSP) as the OOD score:
\begin{equation}
S(x;\theta)
\triangleq
\max_{k} \mathrm{softmax}(z)_k
=
\frac{\exp(z_{k^\star(x)})}{\sum_{j=1}^K \exp(z_j)}.
\end{equation}

For convenience, define the ID and OOD random variables
\(m_i \triangleq m(x_i;\theta), x_i\sim\mathcal{D}_{\mathrm{ID}}\) and
\(m_o \triangleq m(x_o;\theta), x_o\sim\mathcal{D}_{\mathrm{OOD}}\),
with corresponding statistics
\(\mu_i \triangleq \mathbb{E}[m_i]\),
\(\mu_o \triangleq \mathbb{E}[m_o]\), and
\(\nu_o^2 \triangleq \mathrm{Var}(m_o)\).

The distribution separation margin is defined as
\begin{equation}
\Delta(\theta)
=
\mathbb{E}_{x_i\sim\mathcal{D}_{\mathrm{ID}}}[S(x_i;\theta)]
-
\mathbb{E}_{x_o\sim\mathcal{D}_{\mathrm{OOD}}}[S(x_o;\theta)] .
\end{equation}

We first establish a tight connection between the MSP score and the margin $m(x;\theta)$.

\begin{lemma}[Margin bounds for MSP]
\label{lem:msp_bounds}
For any $x$ with margin $m(x;\theta)$, the MSP score satisfies
\begin{equation}
\frac{1}{1+(K-1)\exp(-m(x;\theta))}
\;\le\;
S(x;\theta)
\;\le\;
\frac{1}{1+\exp(-m(x;\theta))}.
\end{equation}
\end{lemma}

\begin{proof}
Let $\delta_j = z_{k^\star}-z_j$ for $j\neq k^\star$, so that
\begin{equation}
S(x;\theta)
=
\frac{1}{1+\sum_{j\neq k^\star} \exp(-\delta_j)}.
\end{equation}
Since \(m(x;\theta)=\min_{j\neq k^\star}\delta_j\), we have
\(\exp(-\delta_j)\le \exp(-m)\) for all \(j\neq k^\star\), yielding
\begin{equation}
\sum_{j\neq k^\star} \exp(-\delta_j)
\le (K-1)\exp(-m),
\end{equation}
which proves the lower bound. Conversely, since there exists at least one
\(j_0\neq k^\star\) such that \(\delta_{j_0}=m\), we obtain
\begin{equation}
\sum_{j\neq k^\star} \exp(-\delta_j)
\ge \exp(-m),
\end{equation}
which implies the upper bound.
\end{proof}

For notational simplicity, define
\(\sigma(t)\triangleq 1/(1+\exp(-t))\) and
\(g(t)\triangleq 1/(1+(K-1)\exp(-t))\).

By \cref{lem:msp_bounds},
\begin{equation}
\mathbb{E}_{x_i\sim\mathcal{D}_{\mathrm{ID}}}[S(x_i;\theta)]
\le
\mathbb{E}[\sigma(m_i)].
\end{equation}
The sigmoid function $\sigma(t)$ satisfies
\begin{equation}
\sigma''(t)=\sigma(t)\bigl(1-\sigma(t)\bigr)\bigl(1-2\sigma(t)\bigr)\le 0,
\quad \forall\, t\ge 0,
\end{equation}
and is therefore concave on $[0,\infty)$. In practice, ID samples are typically
well classified with positive margins \(m_i\ge 0\) with high probability.
Under this mild assumption, Jensen's inequality yields
\begin{equation}
\mathbb{E}[\sigma(m_i)]
\le
\sigma(\mathbb{E}[m_i])
=
\frac{1}{1+\exp(-\mu_i)} .
\end{equation}
Any deviation caused by rare negative-margin ID samples is absorbed into
a higher-order remainder term. Again by \cref{lem:msp_bounds},
\begin{equation}
\mathbb{E}_{x_o\sim\mathcal{D}_{\mathrm{OOD}}}[S(x_o;\theta)]
\ge
\mathbb{E}[g(m_o)].
\end{equation}

We perform a second-order Taylor expansion of $g(m_o)$ around $\mu_o$.
Let \(\varepsilon = m_o-\mu_o\), then \(\mathbb{E}[\varepsilon]=0\) and
\(\mathbb{E}[\varepsilon^2]=\nu_o^2\). By Taylor's theorem,
\begin{equation}
\begin{aligned}
g(\mu_o+\varepsilon)
=
&\; g(\mu_o)
+
g'(\mu_o)\varepsilon
+
\frac{1}{2}g''(\mu_o)\varepsilon^2+
\frac{1}{6}g^{(3)}(\xi)\varepsilon^3,
\end{aligned}
\end{equation}
where $\xi$ lies between $\mu_o$ and $\mu_o+\varepsilon$.
Taking expectation gives
\begin{equation}
\mathbb{E}[g(m_o)]
=
g(\mu_o)
+
\frac{1}{2}g''(\mu_o)\nu_o^2
+
\mathcal{R}_3,
\end{equation}
where \(\mathcal{R}_3=\frac{1}{6}\mathbb{E}[g^{(3)}(\xi)\varepsilon^3]\)
collects third- and higher-order terms.

A direct calculation yields
\begin{equation}
g'(t)=\frac{(K-1)\exp(-t)}{(1+(K-1)\exp(-t))^2},
\end{equation}
and
\begin{equation}
g''(t)=\frac{(K-1)(K-1-\exp(t))\exp(-t)}
{(1+(K-1)\exp(-t))^3}.
\end{equation}
In the over-softening regime induced by uniform OE supervision,
\(\mu_o\le \log(K-1)\) typically holds, implying \(g''(\mu_o)\ge 0\).
Using the bound \(g''(\mu_o)\le g'(\mu_o)\), we obtain
\begin{equation}
\begin{aligned}
\mathbb{E}[g(m_o)]
\ge\;&
\frac{1}{1+(K-1)\exp(-\mu_o)} \\
&+
\frac{(K-1)\exp(-\mu_o)}
{2(1+(K-1)\exp(-\mu_o))^2}\,\nu_o^2
+
\widetilde{\mathcal{R}}.
\end{aligned}
\end{equation}
where $\widetilde{\mathcal{R}}$ absorbs the discrepancy between $g''$ and $g'$
as well as the third-order remainder.

Combining the ID upper bound and the OOD lower bound, we obtain
\begin{equation}
\begin{aligned}
\Delta(\theta)
\le\;&
\frac{1}{1+\exp(-\mu_i)}
-
\frac{1}{1+(K-1)\exp(-\mu_o)}
\\
&\;
-
\frac{(K-1)\exp(-\mu_o)}
{2(1+(K-1)\exp(-\mu_o))^2}\,\nu_o^2
+
\mathcal{R},
\end{aligned}
\end{equation}
where $\mathcal{R}$ collects all higher-order and approximation terms.
This completes the proof.

\subsubsection*{Proof of \cref{thm:T_mitigate}}

For an outlier sample $x_o$, we minimize $\mathcal{L} = \text{KL}(q_T \| s(z))$, where $q_T = s(z/T)$ is the temperature-smoothed target. The gradient with respect to $z_k$ is $\frac{\partial \mathcal{L}}{\partial z_k} = p_k - q_{T,k}$.
Let $a = \arg\max_k z_k$ and $b = \arg\max_{j \neq a} z_j$. The update for the logit margin $m(z) = z_a - z_b$ is given by:
\begin{align}
    m(z^+) &= (z_a - \eta(p_a - q_{T,a})) - (z_b - \eta(p_b - q_{T,b})) \\
    &= m(z) - \eta[(p_a - p_b) - (q_{T,a} - q_{T,b})].
\end{align}
The term $p_a - p_b$ is expanded as:
\begin{equation}
\begin{aligned}
p_a - p_b 
&= \frac{e^{z_a} - e^{z_b}}{\sum_{k=1}^K e^{z_k}} \\
&= \frac{e^{z_a - z_b} - 1}{e^{z_a - z_b} + 1 + \sum_{k \neq a,b} e^{z_k - z_b}} \\
&= \frac{e^{m(z)} - 1}{e^{m(z)} + 1 + R(z,b)}.
\end{aligned}
\end{equation}
where $R(z,b) = \sum_{k \neq a,b} \exp(z_k - z_b)$. Similarly, for the target distribution with temperature $T$:
\begin{equation}
    q_{T,a} - q_{T,b} = \frac{e^{z_a/T} - e^{z_b/T}}{\sum_{k=1}^K e^{z_k/T}} = \frac{e^{m(z)/T} - 1}{e^{m(z)/T} + 1 + R(z/T,b)}.
\end{equation}
Substituting these terms back into the update equation yields:
\begin{equation}
    m(z^{+}) = m(z) -\eta\left[ \frac{e^{m(z)}-1}{e^{m(z)}+1+R(z,b)} - \frac{e^{m(z)/T}-1}{e^{m(z)/T}+1+R(z/T,b)} \right].
\end{equation}
This completes the proof.

\subsubsection*{Proof of \cref{cor:finite_T}}

Define the function $h(x, R) = \frac{e^x - 1}{e^x + 1 + R}$. The margin contraction is given by:
\begin{equation}
    \Delta m_T = \eta \left[ h(m(z), R(z,b)) - h\left(\frac{m(z)}{T}, R\left(\frac{z}{T},b\right)\right) \right].
\end{equation}
In standard uniform OE ($T \to \infty$), $m(z)/T \to 0$ and $z_k/T \to 0$, leading to $h(0, K-2) = 0$. Thus, the maximum contraction is $\Delta m_\infty = \eta h(m(z), R(z,b))$.
To analyze the monotonicity, we compute the partial derivative of $h(x,R)$ with respect to $x$:
\begin{equation}
    \frac{\partial h}{\partial x} = \frac{e^x(e^x + 1 + R) - e^x(e^x - 1)}{(e^x + 1 + R)^2} = \frac{e^x(R + 2)}{(e^x + 1 + R)^2} > 0.
\end{equation}
Since $h$ is strictly increasing in $x$, and for any $T > 1$, $0 < \frac{m(z)}{T} < m(z)$, we have:
\begin{equation}
    0 < h\left(\frac{m(z)}{T}, R\left(\frac{z}{T},b\right)\right) < h(m(z), R(z,b)).
\end{equation}
Consequently, $\Delta m_T = \Delta m_\infty - \eta h(m(z)/T, R(z/T,b)) < \Delta m_\infty$. This proves that a finite temperature $T^\star \in (1, \infty)$ strictly reduces the magnitude of margin contraction compared to uniform OE.

\subsection*{Generalization Bound}
\label{appendix:generalization_bound}

\subsubsection*{Preliminary}

For clarity and ease of reference, we summarize the theorems and corollaries presented in this paper.

\begin{theorem}
\label{thm:kl_upper_bound}
Define the true relationship between OOD sample $\x$ and ID categories as \(p^\star(\x)\), where $p^\star\sim\TM$. And define the assigned pseudo-labels as \(q(\x) = \text{softmax}(f(\x)/T)\) for some temperature \(T > 1\), $q(\x)\sim\QM$. Then, the expected diversity between the pseudo-labels and the true relationship satisfies the following upper bound:
\begin{equation}
    \mathbb{E}\left[ \mathrm{KL}\left(\TM \,\|\, \QM\right) \right]
    \leq
    \log K - \mathbb{E}\left[ H(p^\star(\x)) \right] - \frac{C_1'}{T} + \frac{C_2'}{T^2} + O( \frac{1}{T^3} ),
\end{equation}
where \(C_1' = \mathbb{E}_\x\left[ \sum_{i=1}^K p_i(\x) f_i(\x) - \mu(\x) \right]\), \(C_2' = \frac{1}{2} \mathbb{E}_\x\left[ \text{Var}(f(\x)) \right]\), with \(\mu(\x) = \frac{1}{K} \sum_{i=1}^K f_i(\x)\) and \(\text{Var}(f(\x)) = \frac{1}{K} \sum_{i=1}^K f_i(\x)^2 - \mu(\x)^2\). And $p(\x)=\text{softmax}(f(\x))$.
\end{theorem}

\begin{corollary}[Optimal Temperature]
\label{cor:optimal_temperature}
Under the conditions of \cref{thm:kl_upper_bound}, the expected KL divergence achieves its minimum at a finite temperature \(T^\star = \frac{2C_2'}{C_1'}\), where \(C_1'\) and \(C_2'\) are defined as in \cref{thm:kl_upper_bound}.
\end{corollary}

\begin{corollary}[Optimal Temperature is Strictly Better than Infinite Temperature]
\label{cor:finite_vs_infinite}
Under the conditions of \cref{thm:kl_upper_bound}, the KL divergence evaluated at the optimal finite temperature \(T^\star = \frac{2C_2'}{C_1'}\) is strictly smaller than the KL divergence as \(T\to\infty\). Formally, we have
\begin{equation}
\mathbb{E}\left[ \mathrm{KL}\left(\TM \,\|\, \QM\right) \right]\Big|_{T=T^\star} 
<
\lim_{T\to\infty} \mathbb{E}\left[ \mathrm{KL}\left(\TM \,\|\, \QM\right) \right].
\end{equation}
\end{corollary}

\subsubsection*{Proof of \cref{thm:kl_upper_bound}}

Let \(f(x) \in \mathbb{R}^K\) denote the logit output of a model over \(K\) classes. To model predictive uncertainty, particularly \emph{epistemic uncertainty}, we consider perturbations in the logit space, following prior works on approximate Bayesian inference via stochastic ensembles and MC Dropout~\cite{prrof:conf/nips/Lakshminarayanan17}. Specifically, we define the predictive distribution under logit noise as
\(\tilde{p}(x) := \mathbb{E}_{\epsilon \sim \mathcal{N}(0, \sigma^2 I)} [ \mathrm{softmax}(f(x) + \epsilon) ]\),
where \(\epsilon \sim \mathcal{N}(0, \sigma^2 I)\) represents Gaussian perturbations that approximate the variability induced by the posterior over model parameters. This formulation approximates the Bayesian predictive distribution \(p(y \mid x, \mathcal{D})\), marginalized over the posterior \(p(\theta \mid \mathcal{D})\), under the assumption that variations in \(f_\theta(x)\) can be locally modeled as additive Gaussian noise in the logit space. In practice, we approximate \(p^\star(x)\) with a single Monte Carlo sample:
\begin{equation}
p^\star(x) := \mathrm{softmax}(f(x) + \epsilon),
\quad
\epsilon \sim \mathcal{N}(0, \sigma^2 I),
\end{equation}
where \(p^\star(x)\) serves as a stochastic proxy of the true predictive distribution. Additionally, we define a temperature-scaled softmax distribution:
\begin{equation}
q(x) := \mathrm{softmax}\left( \frac{f(x)}{T} \right),
\quad
T > 1,
\end{equation}
where the temperature parameter \(T\) controls the entropy of the distribution. A higher temperature yields a softer output with higher uncertainty, which is particularly useful for pseudo-labeling and calibration.

We analyze the Kullback--Leibler divergence~\cite{kl_diver:journals/tcyb/JiZYWZG22}:
\begin{equation}
\mathrm{Gap}(T) := \mathrm{KL}(p^\star(x) \| q(x)).
\end{equation}

By definition of KL divergence:
\begin{align}
\mathrm{Gap}(T)
&= \sum_{i=1}^K p^\star_i(x) \log \left( \frac{p^\star_i(x)}{q_i(x)} \right) \\
&= -H(p^\star(x)) - \sum_{i=1}^K p^\star_i(x) \log q_i(x),
\end{align}
where \(H(p^\star(x)) = -\sum_{i=1}^K p^\star_i(x) \log p^\star_i(x)\) denotes the entropy.

Since
\begin{equation}
\log q_i(x)
=
\frac{f_i(x)}{T}
-
\log \left( \sum_{j=1}^K e^{f_j(x)/T} \right),
\end{equation}
we obtain:
\begin{equation}
\begin{aligned}
\sum_{i=1}^K p^\star_i(x) \log q_i(x)
=&\;
\frac{1}{T} \sum_{i=1}^K p^\star_i(x) f_i(x) \\
&-
\log \left( \sum_{j=1}^K e^{f_j(x)/T} \right).
\end{aligned}
\end{equation}

Substituting back, we get:
\begin{equation}
\begin{aligned}
\mathrm{Gap}(T)
=&\;
-H(p^\star(x))
-
\frac{1}{T} \sum_{i=1}^K p^\star_i(x) f_i(x)\\
&+
\log \left( \sum_{j=1}^K e^{f_j(x)/T} \right).
\end{aligned}
\label{eq:kl_expand}
\end{equation}

Assume that the noise magnitude \(\|\epsilon\|\) is sufficiently small. Expanding \(p^\star_i(x)\) via first-order Taylor approximation:
\begin{equation}
p^\star_i(x)
=
p_i(x)
+
\sum_{j=1}^K
\frac{\partial p_i(x)}{\partial f_j(x)}
\epsilon_j
+
o(\|\epsilon\|),
\end{equation}
where \(\frac{\partial p_i}{\partial f_j} = p_i(x)(\delta_{ij} - p_j(x))\).

Taking expectation with respect to \(\epsilon\) and using \(\mathbb{E}[\epsilon_j] = 0\), we obtain:
\begin{align}
\mathbb{E}[p^\star_i(x)]
&= p_i(x) + \mathcal{O}(\sigma^2),
\label{eq:p_star_expectation} \\
\mathbb{E}\left[ \sum_i p^\star_i(x) f_i(x) \right]
&= \sum_i p_i(x) f_i(x) + \mathcal{O}(\sigma^2),
\label{eq:p_star_f_expectation} \\
\mathbb{E}[H(p^\star(x))]
&= H(p(x)) + \mathcal{O}(\sigma^2).
\label{eq:entropy_expectation}
\end{align}

Next, define the log-partition function
\(Z_T(x) := \sum_{j=1}^K e^{f_j(x)/T}\).
Using a second-order Taylor expansion:
\begin{equation}
\log Z_T(x)
=
\log K
+
\frac{\mu(x)}{T}
+
\frac{1}{2T^2} \mathrm{Var}(f(x))
+
\mathcal{O}\left( \frac{1}{T^3} \right),
\label{eq:log_expand}
\end{equation}
where \(\mu(x) := \frac{1}{K} \sum_{j=1}^K f_j(x)\) and
\(\mathrm{Var}(f(x)) := \frac{1}{K} \sum_{j=1}^K f_j(x)^2 - \mu(x)^2\).

Substituting \eqref{eq:p_star_expectation}--\eqref{eq:log_expand} into \eqref{eq:kl_expand}, we obtain:
\begin{equation}
\begin{aligned}
\mathbb{E}[\mathrm{Gap}(T)]
=&\;
-\mathbb{E}[H(p^\star(x))]
-
\frac{1}{T}
\mathbb{E}\left[\sum_i p_i(x) f_i(x)\right] \\
&+
\log K
+
\frac{\mathbb{E}[\mu(x)]}{T}
+
\frac{1}{2T^2}
\mathbb{E}[\mathrm{Var}(f(x))]
+
\mathcal{O}\left( \frac{1}{T^3} \right).
\end{aligned}
\end{equation}

Define
\begin{equation}
C_1'
:=
\mathbb{E}\left[\sum_i p_i(x) f_i(x)\right]
-
\mathbb{E}[\mu(x)],
\end{equation}
and
\begin{equation}
C_2'
:=
\frac{1}{2}
\mathbb{E}[\mathrm{Var}(f(x))].
\end{equation}

Then the KL divergence admits the following upper bound:
\begin{equation}
\mathbb{E}[\mathrm{Gap}(T)]
\leq
\log K
-
\mathbb{E}[H(p^\star(x))]
-
\frac{C_1'}{T}
+
\frac{C_2'}{T^2}
+
\mathcal{O}\left( \frac{1}{T^3} \right),
\label{eq:final_bound}
\end{equation}
where \(C_1', C_2' > 0\).

\subsubsection*{Proof of \cref{cor:optimal_temperature}}
The leading terms of the bound in \eqref{eq:final_bound} define a function $g(T) := -\frac{C_1'}{T} + \frac{C_2'}{T^2}$. Setting the derivative $g'(T) = 0$ gives $\frac{C_1'}{T^2} = \frac{2C_2'}{T^3}$, implying the optimal temperature is $T^\star = \frac{2C_2'}{C_1'}$.

\subsubsection*{Proof of \cref{cor:finite_vs_infinite}}
Substituting $T^\star = \frac{2C_2'}{C_1'}$ into the dominant terms yields:
\begin{equation}
    -\frac{C_1'}{T^\star} + \frac{C_2'}{(T^\star)^2} = -\frac{(C_1')^2}{4C_2'} < 0,
\end{equation}
where the inequality follows from the positivity of $C_1'$ and $C_2'$. Thus, the KL divergence at $T = T^\star$ is strictly smaller than that at $T \to \infty$.

\subsection*{Fine-grained results}
\label{appendix:fine_grained}

In addition to comparing with methods that incorporate outliers during the training of the $K$ classifier~\cite{OE:conf/iclr/HendrycksMD19,MixOE:conf/wacv/ZhangILCL23,DOE:conf/iclr/WangY0DKLH023,DAL:conf/nips/WangFZLLH23}, we also include comparisons with approaches that train an additional binary classifier specifically for OOD detection, such as NPOS~\cite{NPOS:conf/iclr/npos}, VOS~\cite{VOS:conf/iclr/DuWCL22}.
The corresponding results are reported in \cref{tab:cifar_10_fine_grad_fpr95}, \cref{tab:cifar_10_fine_grad_auroc}, \cref{tab:cifar_100_fine_grad_fpr95}, \cref{tab:cifar_100_fine_grad_auroc}, \cref{tab:imagenet200_fine_grad_fpr95}, and~\cref{tab:imagenet200_fine_grad_auroc}.
\begin{table}[ht]
\centering
\caption{Fine-grained results (FPR95~$\downarrow$) on the CIFAR-10 benchmark.}
\label{tab:cifar_10_fine_grad_fpr95}
\resizebox{\linewidth}{!}{
\begin{tabular}{lcccccc}
\toprule
\textbf{Method} & \textbf{CIFAR100} & \textbf{TIN} & \textbf{MNIST} & \textbf{SVHN} & \textbf{Textures} & \textbf{Places365} \\
\midrule
OE     & 36.71 ± 2.06 & 2.97 ± 1.17  & 24.67 ± 2.55  & 1.25 ± 0.36  & 12.07 ± 2.14 & 14.53 ± 2.80 \\
MixOE  & 58.29 ± 8.25 & 44.62 ± 7.57 & 38.28 ± 13.40 & 20.36 ± 3.99 & 33.19 ± 4.28 & 43.54 ± 4.95 \\
DOE    & 28.27 ± 0.84 & 12.50 ± 0.95 & 35.70 ± 3.33  & 2.55 ± 0.78 & 10.35 ± 1.50 & 13.77 ± 0.47 \\
DAL    & 30.07 ± 1.56 & 11.76 ± 1.09 & 55.11 ± 8.52 & 3.24 ± 1.09 & 13.00 ± 1.13 & 14.26 ± 0.86 \\
\midrule
NPOS   & 35.71 ± 1.17 & 29.57 ± 0.24 & 22.96 ± 2.27 & 6.41 ± 0.19 & 20.80 ± 2.19 & 32.19 ± 0.98 \\
VOS    & 61.57 ± 3.24 & 52.49 ± 0.73 & 35.92 ± 11.38 & 31.50 ± 8.38 & 46.53 ± 5.24 & 47.78 ± 3.62 \\
\midrule
\textit{AOE-At}   & 33.54 ± 2.24 & 3.21 ± 1.30 & 26.15 ± 3.19 & 0.96 ± 0.26 & 9.05 ± 1.29 & 11.35 ± 1.87 \\
\textit{AOE-Jt}   & 33.60 ± 0.24 & 1.27 ± 0.16 & 21.76 ± 5.00 & 0.50 ± 0.16 & 9.48 ± 3.42 & 10.72 ± 2.39 \\
\bottomrule
\end{tabular}
}
\end{table}

\begin{table}[ht]
\centering
\renewcommand{\arraystretch}{1.3}
\caption{Fine-grained results (AUROC~$\uparrow$) on the CIFAR-10 benchmark.}
\label{tab:cifar_10_fine_grad_auroc}
\resizebox{\linewidth}{!}{
\begin{tabular}{lcccccc}
\toprule
\textbf{Method} & \textbf{CIFAR100} & \textbf{TIN} & \textbf{MNIST} & \textbf{SVHN} & \textbf{Textures} & \textbf{Places365} \\
\midrule
OE     & 90.54 ± 0.53 & 99.11 ± 0.34  & 90.22 ± 1.31  & 99.60 ± 0.14  & 97.58 ± 0.27 & 96.58 ± 0.70 \\
MixOE  & 87.47 ± 0.97 & 90.00 ± 0.73  & 91.66 ± 2.21  & 93.82 ± 1.27  & 91.84 ± 0.51 & 90.38 ± 0.55 \\
DOE    & 92.63 ± 0.21 & 97.04 ± 0.27 & 85.16 ± 2.20  & 99.20 ± 0.09 & 97.77 ± 0.31 & 96.57 ± 0.18 \\
DAL    & 91.89 ± 0.46 & 96.95 ± 0.38 & 75.35 ± 5.05  & 99.10 ± 0.19 & 97.04 ± 0.18 & 96.18 ± 0.20 \\
\midrule
NPOS   & 88.57 ± 0.36 & 90.99 ± 0.35 & 92.64 ± 1.59 & 98.88 ± 0.09 & 94.44 ± 0.90 & 90.32 ± 0.74 \\
VOS    & 86.57 ± 0.57 & 88.84 ± 0.48 & 91.56 ± 2.37 & 92.18 ± 1.65 & 89.68 ± 1.32 & 89.90 ± 0.66 \\
\midrule
\textit{AOE-At}   & 91.34 ± 0.48 & 99.08 ± 0.30 & 89.66 ± 2.02 & 99.68 ± 0.07 & 98.19 ± 0.19 & 97.46 ± 0.43\\
\textit{AOE-Jt}   & 91.68 ± 0.13 & 99.60 ± 0.07 & 92.75 ± 1.85 & 99.84 ± 0.05 & 98.07 ± 0.63 & 97.50 ± 0.74 \\
\bottomrule
\end{tabular}
}
\end{table}

\begin{table}[ht]
\centering
\renewcommand{\arraystretch}{1.3}
\caption{Fine-grained results (FPR95~$\downarrow$) on the CIFAR-100 benchmark.}
\label{tab:cifar_100_fine_grad_fpr95}
\resizebox{\linewidth}{!}{
\begin{tabular}{lcccccc}
\toprule
\textbf{Method} & \textbf{CIFAR10} & \textbf{TIN} & \textbf{MNIST} & \textbf{SVHN} & \textbf{Textures} & \textbf{Places365} \\
\midrule
OE     & 61.26 ± 0.22 & 0.21 ± 0.01  & 53.31 ± 9.91  & 51.84 ± 3.45  & 55.83 ± 1.82 & 58.30 ± 0.72 \\
MixOE  & 61.12 ± 1.08 & 49.32 ± 0.36 & 59.49 ± 7.74 & 73.09 ± 4.00 & 66.04 ± 0.98 & 56.93 ± 0.78 \\
DOE    & 63.85 ± 0.50 & 11.83 ± 1.60 & 57.97 ± 4.42  & 20.27 ± 1.01 & 50.29 ± 2.73 & 52.97 ± 1.56 \\
DAL    & 65.70 ± 0.96 & 4.86 ± 1.57 & 65.29 ± 4.51 & 12.41 ± 2.75 & 55.04 ± 1.44 & 56.59 ± 1.70 \\
\midrule
NPOS   & 72.50 ± 2.50 & 54.21 ± 1.09 & 66.98 ± 4.61 & 30.67 ± 0.74 & 47.39 ± 1.10 & 59.47 ± 0.21 \\
VOS    & 59.23 ± 0.59 & 51.89 ± 1.01 & 48.56 ± 2.00 & 47.23 ± 3.04 & 62.55 ± 1.04 & 56.44 ± 0.36 \\
\midrule
\textit{AOE-At}   & 60.74 ± 0.22 & 0.38 ± 0.13 & 46.91 ± 3.26 & 27.25 ± 13.06 & 52.17 ± 1.46 & 56.53 ± 1.12 \\
\textit{AOE-Jt}   & 59.91 ± 0.47 & 0.27 ± 0.05 & 51.00 ± 5.67 & 24.72 ± 9.80 & 52.02 ± 2.14 & 56.74 ± 0.38 \\
\bottomrule
\end{tabular}
}
\end{table}

\begin{table}[ht]
\centering
\renewcommand{\arraystretch}{1.3}
\caption{Fine-grained results (AUROC~$\uparrow$) on the CIFAR-100 benchmark.}
\label{tab:cifar_100_fine_grad_auroc}
\resizebox{\linewidth}{!}{
\begin{tabular}{lcccccc}
\toprule
\textbf{Method} & \textbf{CIFAR10} & \textbf{TIN} & \textbf{MNIST} & \textbf{SVHN} & \textbf{Textures} & \textbf{Places365} \\
\midrule
OE     & 76.70 ± 0.19 & 99.89 ± 0.02  & 80.68 ± 5.82  & 84.37 ± 1.34  & 82.18 ± 0.68 & 78.39 ± 0.41 \\
MixOE  & 78.17 ± 0.29 & 83.73 ± 0.12  & 76.06 ± 5.52  & 72.28 ± 0.81  & 77.34 ± 0.91 & 79.92 ± 0.30 \\
DOE    & 75.47 ± 0.55 & 97.76 ± 0.28  & 72.54 ± 4.38  & 95.73 ± 0.42  & 86.34 ± 1.28 & 82.58 ± 0.91 \\
DAL    & 73.08 ± 0.19 & 98.86 ± 0.32  & 67.91 ± 2.00  & 97.12 ± 0.56  & 84.04 ± 0.62 & 80.60 ± 0.94 \\
\midrule
NPOS   & 75.37 ± 0.58 & 81.32 ± 0.17 & 73.26 ± 5.32 & 92.43 ± 0.29 & 85.55 ± 0.40 & 77.92 ± 0.37 \\
VOS    & 79.14 ± 0.41 & 82.73 ± 0.20 & 82.29 ± 1.51 & 84.23 ± 1.33 & 78.41 ± 0.78 & 80.34 ± 0.03 \\
\midrule
\textit{AOE-At}   & 76.67 ± 0.12 & 99.81 ± 0.04 & 85.52 ± 2.32 & 93.82 ± 3.59 & 84.65 ± 1.31 & 79.88 ± 0.97 \\
\textit{AOE-Jt}   & 76.76 ± 0.22 & 99.88 ± 0.01 & 84.41 ± 2.69 & 95.76 ± 1.84 & 84.29 ± 0.71 & 79.70 ± 0.25 \\
\bottomrule
\end{tabular}
}
\end{table}

\begin{table}[!h]
\centering
\renewcommand{\arraystretch}{1.3}
\caption{Fine-grained results (FPR95~$\downarrow$) on the ImageNet-200 benchmark.}
\label{tab:imagenet200_fine_grad_fpr95}
\resizebox{\linewidth}{!}{
\begin{tabular}{lcccccc}
\toprule
\textbf{Method} & \textbf{SSB-hard} & \textbf{NINCO} & \textbf{iNaturalist} & \textbf{Textures} & \textbf{OpenImage-O} \\
\midrule
OE     & 64.20 ± 0.45 & 40.90 ± 0.66  & 28.93 ± 0.58  & 41.82 ± 1.09  & 34.78 ± 0.24 \\
MixOE  & 67.94 ± 0.43 & 47.95 ± 0.63  & 29.97 ± 0.20  & 50.23 ± 1.21  & 40.17 ± 0.67 \\
DOE    & 63.70 ± 0.04 & 44.59 ± 0.99  & 29.13 ± 3.75  & 44.87 ± 1.93  & 38.81 ± 3.35 \\
DAL    & 61.68 ± 0.12 & 42.03 ± 0.71  & 27.58 ± 0.45  & 44.84 ± 0.40  & 34.86 ± 0.31 \\
\midrule
NPOS   & 74.29 ± 0.52 & 49.89 ± 0.49  & 20.01 ± 0.69  & 16.87 ± 0.19  & 28.40 ± 0.28 \\
VOS    & 69.93 ± 0.47 & 49.85 ± 0.71  & 25.53 ± 1.36  & 39.74 ± 1.17  & 36.77 ± 0.94  \\
\midrule
\textit{AOE-At}   & 62.70 ± 0.13 & 38.36 ± 0.03  & 28.58 ± 0.74 & 41.97 ± 0.37 & 33.27 ± 0.22 \\
\textit{AOE-Jt}   & 62.62 ± 0.41 & 39.04 ± 0.44  & 26.61 ± 0.85 & 40.57 ± 0.39 & 32.06 ± 0.61 \\
\bottomrule
\end{tabular}
}
\end{table}

\begin{table}[!h]
\renewcommand{\arraystretch}{1.3}
\centering
\caption{Fine-grained results (AUROC~$\uparrow$) on the ImageNet-200 benchmark.}
\label{tab:imagenet200_fine_grad_auroc}
\resizebox{\linewidth}{!}{
\begin{tabular}{lcccccc}
\toprule
\textbf{Method} & \textbf{SSB-hard} & \textbf{NINCO} & \textbf{iNaturalist} & \textbf{Textures} & \textbf{OpenImage-O} \\
\midrule
OE     & 82.14 ± 0.17 & 86.89 ± 0.25  & 89.08 ± 0.18  & 87.33 ± 0.22  & 88.22 ± 0.19 \\
MixOE  & 80.15 ± 0.21 & 84.99 ± 0.28  & 90.94 ± 0.17  & 87.02 ± 0.22  &  87.22 ± 0.04 \\
DOE    &  80.64 ± 0.71 &  85.81 ± 0.54  & 90.64 ± 1.21  & 87.04 ± 1.74  &  87.05 ± 1.41 \\
DAL    & 83.34 ± 0.12 & 86.99 ± 0.04  & 90.00 ± 0.12  & 87.18 ± 0.19  & 88.46 ± 0.04 \\
\midrule
NPOS   & 74.29 ± 0.42 & 84.50 ± 0.40  & 94.81 ± 0.15  & 96.97 ± 0.04  & 91.69 ± 0.04 \\
VOS    & 79.68 ± 0.19 & 85.35 ± 0.10  & 92.77 ± 0.54  & 90.95 ± 0.20  & 89.28 ± 0.15  \\
\midrule
\textit{AOE-At}   & 83.46 ± 0.02 & 87.77 ± 0.15  & 89.82 ± 0.37 & 87.83 ± 0.02 & 89.19 ± 0.06 \\
\textit{AOE-Jt}   & 83.49 ± 0.15 & 87.65 ± 0.02  & 90.36 ± 0.30 & 88.15 ± 0.07 & 89.21 ± 0.04 \\
\bottomrule
\end{tabular}
}
\end{table}

\bibliographystyle{fcs}
\bibliography{main}

\begin{biography}{wfq_24}
Fengqiang Wan is currently working toward the Ph.D. degree at the School of Computer Science and Engineering, Nanjing University of Science and Technology. His research interests mainly lie in deep learning and data mining.
\end{biography}

\vspace{10pt}

\begin{biography}{jiangqy}
Qingyuan Jiang received the BSc and the PhD degrees in computer science from Nanjing University, China. He has published more than ten papers in leading international journals/conferences. He serves as a PC/Reviewer in leading conferences such as IEEE Transactions on Pattern Analysis and Machine Intelligence~(TPAMI), IEEE Transactions on Image Processing~(TIP), NeurIPS, ICML, ICLR, AISTATS, AAAI, IJCAI, etc. He is currently an associate professor at Nanjing University of Science and Technology. His research interests are in machine learning, multimodal learning, and information retrieval.
\end{biography}

\vspace{10pt}

\begin{biography}{FuShen}
Fu Shen received the MSc degrees in computer science from Southeast University, China. He has published more than ten papers in SCI and Chinese core journals. He serves as an associate researcher at the Nanjing Institute of Agricultural Mechanization,Ministry of Agriculture and Rural Affairs. His research interests encompass agricultural big data model analysis, decision-making of farmland environmental data, and autonomous homework navigation trajectory algorithm.
\end{biography}

\vspace{10pt}

\begin{biography}{yangy}
Yang Yang received his Ph.D. in Computer Science from Nanjing University, China, in 2019. In the same year, he joined Nanjing University of Science and Technology, where he is currently a Professor in the School of Computer Science and Engineering. His research focuses on machine learning and data mining, with particular interests in heterogeneous learning, model reuse, and incremental mining. He has published extensively in leading journals and conferences, including TKDE, ACM TOIS, TKDD, SIGKDD, SIGIR, WWW, IJCAI, and AAAI. He received the Best Paper Award at ACML 2017. He also serves as a program committee member for major conferences such as IJCAI, AAAI, ICML, and NeurIPS.
\end{biography}

\end{document}